\documentclass[10pt,reqno]{amsart}

\usepackage{amsthm,amsmath,amsfonts,amssymb}
\usepackage{bm}
\usepackage[dvipdfmx]{graphicx}
\usepackage{xcolor}
\usepackage{fullpage}
\usepackage{verbatim}
\usepackage{caption}
\usepackage{float}

\usepackage{tikz}
\usetikzlibrary{calc, positioning}
\usetikzlibrary{decorations.markings}
\usetikzlibrary{quotes,angles}
\usetikzlibrary{arrows.meta}
\usetikzlibrary{intersections}
\usetikzlibrary{graphs}

\usepackage{appendix}
\usepackage[sort&compress, numbers]{natbib}

\usepackage{hyperref}					
\hypersetup{colorlinks,					
	linkcolor=blue,
	citecolor=red}

\usepackage{amsaddr}
\usepackage{mathrsfs}

\newtheorem{theorem}{Theorem}

\newtheorem{lemma}{Lemma}
\newtheorem{proposition}{Proposition}
\newtheorem{assumption}{Assumption}
\theoremstyle{definition}
\newtheorem{definition}{Definition}
\newtheorem{remark}{Remark}

\def\R{\mathbb{R}}
\def\E{\mathbb{E}}

\def\Z{\mathbb{Z}}

\def\AA{\mathcal{A}}
\def\BB{\mathcal{B}}
\def\MM{\mathcal{M}}

\def\FF{\mathcal{F}}
\def\LL{\mathcal{L}}

\def\GG{\mathcal{G}}
\def\NN{\mathcal{N}}

\def\LD{\mathscr{D}}
\def\PLDR{\mathcal{M}} 

\def\ord{\mathrm{ord}}

\newcommand{\wt}{\widetilde}

\newcommand{\Cov}{\operatorname{Cov}}
\renewcommand{\d}{\mathrm{d}}

\def\llan{\left\langle}
\def\rran{\right\rangle}

\allowdisplaybreaks

\title{The Differential Neural Tangent Kernel and Its Positivity}
\author{ Bangti Jin\and Longjun Wu} 
\address{ Department of Mathematics, The Chinese University of Hong Kong, Shatin, N.T., Hong Kong, P.R. China }

\begin{document}

\maketitle

\begin{abstract}
The Neural Tangent Kernel (NTK) is one powerful tool for analyzing the training dynamics  
of neural networks in the over-parameterized regime. Recently, the theoretical framework has been extended to physics-informed neural networks (PINNs) for solving linear PDEs, one highly popular class of neural PDE solvers. In the analysis, the positivity of the associated NTK plays a fundamental role. However, establishing the positivity of the NTK for PINNs is highly challenging, due to the presence of multiple differential operators. 
In this work, we propose a new theoretical framework, called Differential Neural Tangent Kernel (DNTK), for analyzing PINNs through the lens of the NTK, and establish the positivity of the infinite width DNTK for both shallow and deep neural networks for a wide class of activation functions, including RePU and smooth but non-polynomial activations, for all linear differential operators. These theoretical results lay the foundation for the analysis of gradient type algorithms for training PINNs. 
\end{abstract}

\section{Introduction}

In recent years, deep neural networks (DNNs) have been established as an effective tool to solve high-dimensional Partial 
Differential Equations (PDEs) \citep{HanJentzenE:2018,KhooLu:2021}, among which physics informed neural networks (PINNs) \citep{Raissi:2019PINN} (see \citep{Dissanayake:1994,lagaris:1998artificial,Lagaris:2000,Sirignano:2018DGM} for related contributions) represent one of the most popular classes of neural PDE solvers. Due to their ease of implementation and flexibility with various PDE models, PINNs have been widely employed for solving a wide variety of PDE problems arising in scientific and engineer disciplines, e.g., Poisson equation with singularities \citep{HuJinZhou:2024}, fluid mechanics \citep{JinCai:2021,Eivazi:2022}, solid mechanics \citep{Sirignano:2018DGM}, 
and various PDE inverse problems \citep{JinLiQuan:2024,CenJinLi:2025,Tanyu2023}, and have exhibited strong empirical performance in practice. We refer interested readers to the reviews \citep{DeRyckMishra:2024,ToscanoKarniadakis:2025} for further details on the practical applications, methodological developments and numerical analysis of PINNs. Thus there is enormous interest in providing the theoretical analysis of PINNs, and important progresses have been made in terms of the analysis of the approximation error and statistical error; see, e.g.,\citep{Shin:2020Convergence,JiaoLaiLu:2022cicp,Zeinhofer:2025,DoumecheBoyer:2025,ChengChenLin:2025} for an incomplete list and the review \cite{DeRyckMishra:2024} and the references therein. 

In practice, PINNs are often trained using gradient type algorithms, e.g., (stochastic) gradient descent, Adam \citep{Kingma:2017} and L-BFGS \citep{ByrdLu:1995}, and these algorithms often perform reasonably well. Despite the impressive empirical successes, the theoretical analysis of training processes is formidably challenging. Indeed, due to the nonlinearity of the DNNs with respect to the DNN parameters, the PINN loss functions are highly nonconvex with respect to the DNN parameters and fraught with many bad local minima, and it is very challenging to prove that the chosen training algorithm converges to a global minimizer of the loss. This issue represents one outstanding challenge in the convergence analysis of PINNs. One highly powerful idea to prove the convergence of gradient type algorithm for regression tasks involving DNNs is the neural tangent kernel (NTK) framework due to \cite{Jacot:2018NTK}. In essence, in the NTK regime, the training dynamics of the DNNs is approximated by that of analytically more tractable kernel regression. The general strategy is to prove that the NTK of DNNs is positive definite at (random) initialization and remains largely constant (and thus stays positive definite) during the entire optimization process in the infinite-width case, which linearizes the DNN and allows attaining global optima \citep{Jacot:2018NTK}. In the last few years, within the NTK framework, several global convergence results of gradient descent type algorithms for wide two-layer PINNs have been successfully established \citep{Gao:2023GD,Xu:2024IGD,Xu:2024NGD,JinWu:2025}. \cite{Gao:2023GD} established the convergence of vanilla GD for training two layer PINNs for the heat equation. \cite{Xu:2024IGD} generalized the convergence result to implicit gradient descent, and later \cite{Xu:2024NGD} extended the result to natural gradient descent. Very recently, \cite{JinWu:2025} 
proved the convergence for stochastic gradient descent and stochastic gradient flow for the standard Poisson equation, one model second-order elliptic equation.

In the NTK framework, the positivity of the NTK plays an essential role in establishing the convergence guarantee. The positivity of the NTK for the regression task is by now well-established in various settings, and  \cite{Carvalho:2025Positivity} provided a nearly sharp positivity result under very generous assumptions on the activation function and training data; see Section \ref{ssec:NTK-DNTK} below for more detailed discussions.  In sharp contrast, the study of the NTK for PINNs is much more challenging due to the presence of (multiple) differential operators and relevant positivity results are still very scarce. To the best of our knowledge, the existing results on the positivity remain limited to two-layer neural networks (NNs) for PDEs with one boundary condition \citep{Gao:2023GD,Xu:2024IGD,ZhaoLuo:2025}.
\cite{Gao:2023GD} showed the positivity of infinite width NTK for the heat equation with the ReLU\(^3\) activation by adapting the technique from the work \cite{Du:2019shallow}. 
\cite{Xu:2024IGD} proved the positivity for any smooth non-polynomial activation, using the key observation that the high order tensor product of pairwise non-parallel vectors are linearly independent. \cite{ZhaoLuo:2025} generalized the result to a class of linear differential operators, possibly high-order, with the \(\mathrm{tanh}\) activation. 
In addition, several works investigated the training dynamics under the positivity assumption of the NTK for PINNs. \cite{WangWang:2021eigen} studied the spectral bias of the NTK numerically and proposed a new architecture for PINNs. \cite{LuoYang:2024} incorporated the optimization error into the generalization analysis. \cite{GanLiLin:2025NTK} developed the NTK theory for DNNs with one fixed linear operator and gave the convergence of the network trained by gradient descent under the assumption of the positivity of the kernel matrix.

The preceding discussions underscore the imperative need for rigorously establishing the positive definiteness of the NTK for PINNs using DNNs, for a broad class of activation functions. However, the existing studies on the positivity still have some limitations. \cite{Gao:2023GD} only dealt with the RePU (i.e., powers of the ReLU) activation while 
\cite{Xu:2024IGD} treated smooth activations but only for the heat equation. 
\cite{ZhaoLuo:2025} generalized the argument to a certain class of admissible linear operators with \(\mathrm{tanh}\) activation, which however does not include frequently encountered  operators like elliptic, parabolic and hyperbolic operators. Note that the proof method in \cite{Gao:2023GD} is tailored to the RePU activation, while the method in \cite{Xu:2024IGD,ZhaoLuo:2025} is limited in its applicability to the PDE with tractable structure and may fail to handle complex PDE operators. 
Moreover, these studies are confined to two operators, i.e., PDE and one boundary operator, while in practice, multiple operators are abundant in PDE problems in many areas of science and engineering, like initial boundary value problems and high-order boundary value problems, e.g., polyharmonic operators.  Thus it is of great importance to develop a systematic theoretical framework that can address a wide range of (possibly high-order) linear PDE problems which may involve multiple differential operators. 

In this work, we systematically investigate the positivity issue of the NTK for PINNs. We 
propose the concept the Differential Neural Tangent Kernel (DNTK) as a general theoretical framework for the positivity analysis. Our analysis of the NTK for PINNs covers a broad range of practically important settings, including multiple differential operators, shallow and deep NNs, and with or without bias terms, for a wide range of activations. Specifically, we establish the infinite width limits of the NN outputs and the DNTK, with explicit expressions and recursion formulas for the associated covariance functions for the shallow and deep NNs, respectively, and develop a new approach to analyze the positivity of the DNTK for two-layer NNs and also obtain the positivity result for DNNs. These results not only subsume the sharp result on the positivity for NTK established in \cite{Carvalho:2025Positivity}, but also cover all the existing results on the positivity of the NTK for two-layer PINNs \citep{Gao:2023GD, Xu:2024IGD, ZhaoLuo:2025}; see Theorems \ref{T:infinite-limit} and \ref{T:cor_Pos_DNTK} in Section \ref{S:DNTK-Def} for the precise statements of the main results for DNNs. To the best of our knowledge, there is no prior work on the positivity of the NTK for PINNs with DNNs in the literature. The main technical tools in the positivity analysis of the DNTK include $\sigma$-polynomials (in the standard form) in Definition \ref{D:sigma-poly} and the property (LI) in Definition \ref{D:Property (LI)} arising in the analysis for shallow and deep neural networks, respectively. The overall strategy is to transform the positivity of the covariance function and the DNTK into the linear independence of \(\sigma\)-polynomials using their explicit expressions for shallow networks, and to finally establish such linear independence. Then, using inductive method to prove the property (LI) for deep networks, from which we obtain the positivity for DNNs. The unified DNTK framework enables examining the impact of differential operators, the presence of the bias term (cf. Theorems \ref{T:LI-smooth_uv} and \ref{T:LI-smo_u}) and the neural network depth (see the proof of Theorem \ref{T:cor_Pos_DNTK}) on the positivity, which remain rather obscure when studying only a single particular PDE. These theoretical results lay the foundation for further study: Following the proof strategies described in \citep{Gao:2023GD,Xu:2024IGD,Xu:2024NGD,JinWu:2025}, one can establish the convergence analysis of gradient type algorithms for training wide PINNs.

The rest of the paper is organized as follows. In Section \ref{S:DNTK-Def}, we give the definition of the DNTK, which generalizes the NTK theory by incorporating differential operators. Then, we treat the DNTK theory for shallow and deep NNs separately.
In Section \ref{S:DNTK-shallow}, we derive the infinite-width limit of the network output, which converges to a Gaussian process, and the DNTK, for two-layer NNs.
In Section \ref{S:DNTK-deep}, we derive crucial recursive formulas of the covariance function and the DNTK for DNNs, 
and establish the positivity result by mathematical induction on the recurrent relation. In the appendices we provide the technical proofs of all the theoretical results. Throughout, for an integer $n$, the notation $[n]$ denotes the set $\{1,\ldots,n\}$.

\section{Differential Neural Tangent Kernel}\label{S:DNTK-Def}

In this section, we propose the concept of differential neural tangent kernel (DNTK) based on the training dynamics of PINNs,  describe the main theoretical results and discuss the differences between the standard NTK and the DNTK.
\subsection{Preliminary}

Given an integer \(L\geq 1\), consider fully-connected NNs with \(L-1\) hidden layers 
\( f = f^{(L)} : \mathbb{R}^{d_0} \to \mathbb{R}^{d_L} \) (with $d_0=d$) defined recursively by the relations
\begin{equation}\label{E:L-layerNN}
\begin{aligned}
	f^{(1)}(x;\theta) & = W^{(1)} x + \gamma b^{(1)},\\
	f^{(\ell+1)}(x;\theta) & = \frac{1}{\sqrt{d_\ell}} W^{(\ell+1)} 
	\sigma \left( f^{(\ell)}(x;\theta) \right) + \gamma b^{(\ell+1)}, \quad \ell=1,\ldots, L-1,
\end{aligned}
\end{equation}
where $ W^{(\ell)} \in \mathbb{R}^{d_\ell \times d_{\ell-1}}$ and $b^{(\ell)} \in \mathbb{R}^{d_\ell}$ are the weight matrix and bias vector at the $\ell$th layer, respectively, \( \sigma : \mathbb{R} \to \mathbb{R} \) is a nonlinear activation function that applies entrywise 
to vectors, and the weight \( \gamma \geq 0 \) acts on the bias and is non-learnable, which can be used to control the intensity of the bias.

Let \(\FF_d\) be the space of all functions from \(\R^d\) to \(\R\) and
\(\LD_d\) the space of all linear differential operators
acting on functions from \(\R^d\) to \(\R\) with the coefficients from \(\mathcal{F}_d\). We denote the Cartesian product space $ \LD_d\times \R^d$ by $\PLDR_d$.  The DNN output \(f^{(L)}\) is viewed as a function on the product space \(\PLDR_d\), given by
\begin{equation}\label{E:DiffOut}
	f^{(L)}:  \PLDR_d \to \R^{d_L}, \quad
	  (\AA, x) \mapsto \AA f^{(L)}(x;\theta) 
	  := \AA_z \left. f^{(L)}(z;\theta) \right|_{z=x},
\end{equation}
where the operator $\mathcal{A}$ acts componentwise on the vector field \(f^{(L)}\). We denote by \(x^\AA\) the point 
\((\AA, x)\in \PLDR_d\), called the \emph{\(\AA\)-marked point \(x\)}. 
When \(\AA\) is fixed, we obtain the differential of the NN \(\AA f^{(L)}:\R^d\to\R^{d_L}\).

\begin{remark}\label{R:non-smooth}
In practice, \(\sigma\) may be not very smooth, and thus the NN may be not smooth. Then one should restrict the maximum order $m$ of the differentiable operators so the PDE holds in a strong sense. For example, if $\sigma$ is the ReLU power , i.e., \(\sigma=\mathrm{ReLU}^{m+1}\), then the domain of \(f^{(L)}\) should be \(\PLDR^{(m)}_d := \LD_d^{(m)} \times \R^d\), where \(\LD_d^{(m)}\subset \LD_d\) denotes the set of linear differential operators of order at most \(m\), with an additional order reserved for taking the derivative with respect to the NN parameters $\theta$. We will omit this subtlety in the following discussions.
\end{remark}

\subsection{Motivation of the DNTK}

The motivation of the DNTK originates from gradient type algorithms for training PINNs. We illustrate the idea by the PDE with one boundary condition, but the analysis below adapts to all linear PDEs with multiple linear boundary or initial conditions.
Let \(\Omega \subset \mathbb{R}^d\) be a bounded domain with a boundary $\partial\Omega$.  Given the linear differential operators \(\LL\) and \(\BB\), consider the following PDE problem: 
\begin{equation*}
	\left\{
	\begin{aligned}
		 \mathcal{L}[u]  &= \phi, \quad \mbox{in }\Omega, \\
		\mathcal{B}[u]& = \psi, \quad \mbox{on }  \partial \Omega.
	\end{aligned}\right.
\end{equation*}
For example, \(\LL\) can be a second-order elliptic  operator while \(\BB\) may represent linear Dirichlet, Neumman or Robin boundary condition.
Given the NN \(f(x;\theta)\), the PINN \citep{Raissi:2019PINN} minimizes the following empirical loss 
\begin{equation}\label{E:PI-loss_empi}
	\begin{aligned}
		L(\theta) := \frac{1}{2}\sum_{p=1}^{N_1} \left( \mathcal{L} f(x_p;\theta)-\phi(x_p)\right)^2
		 + \frac{\lambda^2}{2}\sum_{q=1}^{N_2} \left( \mathcal{B} f(y_q;\theta)-\psi(y_p)\right)^2
	\end{aligned}
\end{equation}
with the training samples \(X = \{x_p\}_{p=1}^{N_1}\) and \(Y = \{{y}_q\}_{q=1}^{N_2}\)
(randomly) drawn from the domain $\Omega$ and the boundary \(\partial\Omega\), respectively, 
and \(\lambda>0\) being a hyperparameter balancing the two terms.
We denote the interior and boundary losses by \(s_p(\theta)\) and \(h_q(\theta)\) respectively:
\begin{gather*}
	s_p(\theta) = \mathcal{L} f(x_p;\theta)-\phi(x_p) \quad\mbox{and}\quad 
	h_q(\theta) = \lambda\left(\mathcal{B} f(y_q;\theta)-\psi(y_p)\right),
\end{gather*}
and the loss vectors by
$s(\theta)=[s_1(\theta), \cdots, s_{N_1}(\theta)]^\top\in \mathbb{R}^{N_1}$ and $h(\theta)=[h_1(\theta), \cdots, h_{N_2}(\theta)]^\top \in \mathbb{R}^{N_2}$.
Often one updates the parameter \(\theta\) by gradient type algorithms, e.g., gradient descent, Adam and L-BFGS. We use the gradient flow, the limit version of vanila gradient descent, to illustrate the idea. Then, we have 
\begin{gather*}
	\frac{\d\theta_t}{\d t} = - \nabla_\theta L(\theta_t)
	= - \sum_{p=1}^{N_1}  s_p(\theta) \nabla_\theta s_p(\theta) 
		- \sum_{q=1}^{N_2}  h_q(\theta) \nabla_\theta h_q(\theta).
\end{gather*}
For \( x_i\in X \), the dynamics of the loss $s_i(\theta_t)$ is given by
\begin{align*}
	 \frac{\d s_i(\theta_t) }{\d t} &= \nabla_\theta s_i(\theta_t) \cdot \frac{\d\theta_t}{\d t} 
	  = -\nabla_\theta s_i(\theta_t)\cdot \left(\sum_{p=1}^{N_1}  s_p(\theta_t) \nabla_\theta s_p(\theta_t)
	   + \sum_{q=1}^{N_2} h_q(\theta_t) \nabla_\theta h_q(\theta_t)\right) \\
	& = - \sum_{p=1}^{N_1} \llan \nabla_\theta \mathcal{L} f(x_i;\theta_t), 
	  \nabla_\theta \mathcal{L} f(x_p;\theta_t) \rran s_p 
	  - \sum_{q=1}^{N_2} \llan \nabla_\theta \mathcal{L} f(x_i;\theta_t), 
	  \nabla_\theta \lambda \mathcal{B} f(y_q;\theta_t) \rran h_q.
\end{align*}
Similarly for \( y_j\in Y \), we have 
\begin{align*}
	 \frac{\d h_j(\theta_t) }{\d t} &= \nabla_\theta h_j(\theta_t) \cdot \frac{\d\theta_t}{\d t} 
	  = -\nabla_\theta h_j(\theta_t)\cdot \left(\sum_{p=1}^{N_1} s_p(\theta_t)\nabla_\theta s_p(\theta_t) 
	    + \sum_{q=1}^{N_2} h_q(\theta_t) \nabla_\theta h_q(\theta_t)\right) \\
	& = - \sum_{p=1}^{N_1} \llan \nabla_\theta \lambda \mathcal{B} f(y_j;\theta_t), 
	  \nabla_\theta \mathcal{L} f(x_p;\theta_t) \rran s_p 
	  - \sum_{q=1}^{N_2} \llan \nabla_\theta \lambda \mathcal{B} f(y_j;\theta_t), 
	  \nabla_\theta \lambda \mathcal{B} f(y_q;\theta_t) \rran h_q.
\end{align*}
By combining these two identities, we obtain
\begin{equation*}
	\begin{aligned}
		\frac{\d}{\d t}   \begin{bmatrix} \mathcal{L}f(X;\theta_t) \\ \lambda\mathcal{B}f(Y;\theta_t) \end{bmatrix} 
		= \frac{\d}{\d t}   \begin{bmatrix} s(\theta_t) \\ h(\theta_t) \end{bmatrix} 
		= - G(\theta_t) \begin{bmatrix} s(\theta_t) \\ h(\theta_t) \end{bmatrix}, 
	\end{aligned}
\end{equation*}
with the Gram matrix \(G(\theta_t)\) given by
\begin{align*}
G(\theta_t) = D(\theta_t)^\top D(\theta_t) 
\in \mathbb{R}^{(N_1 + N_2)\times (N_1 + N_2)},
\end{align*}
with
\begin{align*}
	D(\theta_t) & = \begin{bmatrix}
		 \nabla_\theta s_1(\theta_t)
		& \cdots 
		&  \nabla_\theta s_{N_1}(\theta_t) 
		&  \nabla_\theta h_1(\theta_t)
		& \cdots 
		&  \nabla_\theta h_{N_2}(\theta_t)
		\end{bmatrix} \\
	& = \begin{bmatrix} \nabla_\theta \mathcal{L} f(X;\theta_t) &
	 	\nabla_\theta \lambda\mathcal{B} f(Y;\theta_t) \end{bmatrix}.
\end{align*}
Thus, the Gram matrix $G(\theta_t)$ can be presented by
\begin{align}\label{eqn:gram}
	G(\theta_t)  = 
	\begin{bmatrix}
		\llan \nabla_\theta \LL f(X;\theta_t), \nabla_\theta \LL f(X;\theta_t) \rran
		& \llan \nabla_\theta \LL f(X;\theta_t), \nabla_\theta \lambda\BB f(Y;\theta_t) \rran \\
		\llan \nabla_\theta \lambda\BB f(Y;\theta_t), \nabla_\theta \LL f(X;\theta_t) \rran
		& \llan \nabla_\theta \lambda\BB f(Y;\theta_t), \nabla_\theta \lambda\BB f(Y;\theta_t) \rran
	\end{bmatrix}.
\end{align}
These discussions indicate that during the training process, the differential of the DNN and the loss are controlled by the Gram matrix $G(\theta_t)$. 
If the smallest eigenvalue of $G(\theta_t)$ has a positive lower bound during the training process, 
then the loss $L(\theta_t)$ converges to zero. Such lower bounds can be obtained under the positivity of the infinite width NTK in the `lazy training'
regime \citep{Chizat:2019Lazy}. Note that each entry of the Gram matrix \(G(\theta_t)\) for the PINNs is the inner product between gradients of the differential of NN outputs with respect to the parameters $\theta$. The differential 
operators depend on the type of sampling points, e.g., \(\LL\) for the interior
points \(X\) and \(\lambda\BB\) for the boundary points \(Y\). The structure of the Gram matrix motivates the following definition.

\begin{definition}
Let \(f^{(L)}\) be an $L$-layer NN with \(d_L\) outputs. For any \(\mu,\nu\in[d_L]\), the \textbf{Differential Neural Tangent Kernel} (DNTK) of  \(f^{(L)}\) is defined on the set \(\PLDR_d\) and given by
	\( \Theta^{(L)}_{\mu\nu}: \PLDR_d\times \PLDR_d\to \R \) such that
	\begin{align*}
		\Theta^{(L)}_{\mu\nu}\left(x^\AA, y^\BB\right) 
		= \llan \nabla_\theta \mathcal{A}f^{(L)}_\mu(x;\theta), \nabla_\theta\mathcal{B}f^{(L)}_\nu(y;\theta)  \rran
		= \sum_{\rho \in \theta} \frac{\partial \mathcal{A}f^{(L)}_\mu(x;\theta)}{\partial \rho}
			\frac{\partial \mathcal{B}f^{(L)}_\nu(y;\theta)}{\partial \rho}.
	\end{align*}
    Given the set \(X=\{ x_1^{\AA_1},\ldots,x_N^{\AA_N}\}\subset \PLDR_d\), the 
    kernel matrix \(\Theta^{(L)}_{\mu\nu}(X)\) induced by the DNTK \(\Theta^{(L)}_{\mu\nu}\) is defined by 
    \[ \Theta^{(L)}_{\mu\nu}(X) := \begin{bmatrix}
			\Theta^{(L)}_{\mu\nu}\left(x_i^{\AA_i}, x_j^{\AA_j}\right)
		\end{bmatrix}_{i,j\in[N]} \in \R^{N\times N}.  \]
\end{definition}

Using the DNTK, The Gram matrix \(G(\theta_t)\) in \eqref{eqn:gram} for the second-order elliptic PDE can be succinctly written as 
\( \Theta_f \left( X^{\LL}, Y^{\lambda\BB};\theta_t \right) \), i.e., the DNTK with respect to the NN \(f(x;\theta)\). The preceding discussion indicates that the DNTK codifies the learning dynamics in the output space if the learning is carried out using gradient descent.

\subsection{Main results}
In this section,  we present two main theoretical results: the existence and the recursive formulas of the infinite-width limit of the NN output and DNTK as functions on the product space \(\PLDR_d\), and the positivity of the covariance function of the infinite-width NN output and the DNTK. These results are given in the following two theorems. The notation $\delta_{\mu\nu}$ denotes the Kronecker product, i.e., \(\delta_{\mu\nu}=1\) if \(\mu=\nu\) and zero otherwise. Further theoretical results are given in Sections \ref{S:DNTK-shallow} and \ref{S:DNTK-deep}.

\begin{theorem}\label{T:infinite-limit}
Suppose that the NN \(f^{(L)}\) with \(L\geq 1\) is given by \eqref{E:L-layerNN} such that \(\sigma\) satisfies  Assumption \ref{A:activation} below and all the parameters are initialized to i.i.d. standard Gaussian. 
In the case of infinite width, i.e., \( d_1,\ldots, d_{L-1} \to \infty \) sequentially, the output of the NN \( f_{\mu}^{(L)}:\PLDR_d \to \R \), for \( \mu \in[d_L] \), converges in  distribution to an i.i.d. centered Gaussian process \(f_\infty^{(L)}\) on \(\PLDR_d\) with the covariance function \(\Sigma^{(L)}_{\infty}\) given by \eqref{E:deep-Cov}, while the the DNTK \( \Theta^{(L)}_{\mu\nu}\) for \( \mu,\nu \in [d_L] \)
converges in probability to the deterministic kernel \(\Theta^{(L)}_{\infty}\delta_{\mu\nu}:\PLDR_d \times \PLDR_d \to \R\) given by \eqref{E:deep-DNTK}. 
\end{theorem}

Theorem \ref{T:infinite-limit} extends the NTK theory for regression tasks due to \cite{Jacot:2018NTK} (see also \cite{AroraSu:2019,Yang:2019} for further results) to the context of PINNs (or PDE regression problems). The key difference lies in the domain of the NN output and the DNTK, i.e., the product space \(\PLDR_d = \LD_d\times\R^d\) incorporates differential operators into the associated kernel. This structure greatly facilitates handling multiple differential operators and unifies the kernel formulation across all (linear) PDEs. \cite{GanLiLin:2025NTK} proved the convergence of the NTK for deep wide PINNs, in which the physics-informed loss with one fixed linear differential operator, and proved that the infinite-width covariance function and NTK in the setting are the derivatives of their counterparts for regression tasks. However, the authors did not provide the recursive formulas. In stark contrast, Theorem \ref{T:infinite-limit} presents the recursive formulas for DNNs, which are crucial to the positivity analysis of the DNTK for DNNs, and see also the explicit formulas for two-layer NNs in Theorem \ref{T:2-ifyDNTK} below.

\begin{theorem}\label{T:cor_Pos_DNTK}
Suppose that the NN \(f^{(L)}\) with \(L\geq 2\) is given by \eqref{E:L-layerNN} such that \(\sigma\) satisfies 
Assumption \ref{A:activation} and \(\gamma > 0\).
Let \(X = \{x_1^{\AA_1},\ldots,x_N^{\AA_N}\}\) sampled from \(\PLDR_d\) such that 
\(x_1,\ldots,x_N\) are pairwise distinct and \(x_i\) lies outside the intersection of the zero set of the coefficients of \(\AA_i\) for all \(i\in[N]\).
If \(\sigma\) is non-polynomial smooth or RePU, then the matrices \(\Sigma^{(L)}_{\infty} \left(X\right)\) and \(\Theta^{(L)}_{\infty} \left(X\right)\) are positive definite.
\end{theorem}

Theorem \ref{T:cor_Pos_DNTK} not only subsumes all of the existing results \cite{Gao:2023GD,Xu:2024IGD,ZhaoLuo:2025} on shallow NNs, but also establishes, for the first time, the positive definiteness of the DNTK for DNNs in the presence of multiple differential operators — an important setting that has not been addressed in prior works. Note that the minimum eigenvalue of an NTK matrix is zero if the activation $\sigma$ is a polynomial and the dataset is sufficiently large \cite[Theorem 4.3]{Panigrahi:2020}, and hence the non-polynomial assumption is not spurious. 
To obtain of the positivity for shallow NNs, we first compute the expressions of the infinite width covariance function and the DNTK. Then, we transform the positivity of the kernel matrix into the linear independence of some polynomial-like functions. Finally, we show the independence of
these functions when the sample points \(X\) and the activation \(\sigma\) satisfy mild assumptions. The positivity for the DNTK of DNNs is derived from that in the shallow case by mathematical induction based on the recursive formulas given in Theorem \ref{T:infinite-limit}. 

\subsection{Comparison between NTK and DNTK} \label{ssec:NTK-DNTK}

In this part, we compare the positivity of the DNTK with that of the well established NTK \citep{Jacot:2018NTK} in order to shed further insights. 

First we briefly review the existing literature on the positivity of the NTK. \cite{Jacot:2018NTK} showed the positivity of the NTK for DNNs when the activation function $\sigma$ is Lipschitz continuous and the training data are sampled from the unit sphere. 
\cite{Du:2019deep} proved a positivity result for DNNs, which removes the unit sphere constraint and also allows the absence of the bias term, under the condition that the activation $\sigma$ is an analytic but non-polynomial function. Furthermore, \cite{Du:2019shallow} provided the positivity result for two-layer
ReLU NNs with samples from the whole Euclidean space, which was then extended to multilayer case by \cite{LiYu:2024EDR}. 
Very recently, \cite{Carvalho:2025Positivity} obtained 
a sharp positivity result of the NTK based on a novel characterization of polynomials, under the mild hypothesis that the activation function $\sigma$ is continuous and almost everywhere differentiable but non-polynomial. We also refer interested readers to \cite{Allen:2019,XieLiang:2017,Nguyen:2021Tight,Karhadkar:2024Bounds,Bombari:2022,Cao:2019Towards} for the spectral property of NTK, which also give some partial results about the positivity.

In contrast, for the NTK of PINNs, the presence of (multiple) differential operators poses several distinct challenges to the positivity analysis. In the infinite-width limit, the physics-informed loss varies with the differential operators, cf. \eqref{E:PI-loss_empi}, which obstructs a unified framework for analyzing the NTK of PINNs. \cite{GanLiLin:2025NTK} treated the case of one single operator, and the analysis does not extend directly to the more general case of multiple operators.
To overcome the challenge, we extend of the domain of the NTK from the Euclidean space \(\R^d\) to the product space \(\PLDR_d\) and integrate the information of differential operators and NNs into the definition of the DNTK, which gives a high-level prospective of the NTK for PINNs. This also greatly facilitates the derivation of useful recursive formulas of the DNTK: indeed the recursion for the DNTK for one differential operator actually depends on other monomial operators, cf. \eqref{E:deep-DNTK}. 
For the positivity of the DNTK, there are two distinct challenges arising in the cases of the shallow and deep NNs respectively. In the shallow case, the positivity issue is transformed to the linear independence of some functions related to the activation function $\sigma$ and the distribution of sampling points. However, these functions are far more involved than the case of the NTK in \citep{Carvalho:2025Positivity} due to the presence of differential operators, cf. \eqref{E:LinRelation}. We propose the algebraic tool \(\sigma\)-polynomials (in standard form) in Definition \ref{D:sigma-poly} in order to establish the independence. In the deep case, the recursive formulas of the DNTK are far more complicated compared to the NTK case \citep{Jacot:2018NTK}, which makes it impossible to derive the positivity result by mathematical induction directly as in the NTK case \citep{Carvalho:2025Positivity}. Instead, we utilize the concept Property (LI) in Definition \ref{D:Property (LI)}, based on which we conduct an inductive argument and derive the positivity of the DNTK for DNNs. Furthermore, the most general positivity result for the NTK in \citep{Carvalho:2025Positivity} is covered by Theorem \ref{T:LI-general_uv} and Theorem \ref{T:Property (LI)}, and thus the results herein are also sharp.

\section{DNTK for shallow NNs} \label{S:DNTK-shallow}

In this section, we study the infinite width behavior of the network output and DNTK for two-layer NNs and give the proofs in Appendix \ref{S-A:DNTK-shallow}. 
Below we fix the hyperparameter \(\gamma\) in \eqref{E:L-layerNN} to 1 and consider the following two-layer NN 
\( f^{(2)} : \mathbb{R}^{d} \to \mathbb{R}^{d_2} \):
\begin{equation}\label{E:2-layerNN}
	f^{(2)}(x;\theta) = \frac{1}{\sqrt{d_1}} W^{(2)} 
	\sigma \left( W^{(1)}x +  b^{(1)} \right) +  b^{(2)}, \quad x\in\R^{d},
\end{equation}
with \(d_1\) representing the width of hidden layer and $W^{(2)} \in \mathbb{R}^{d_2 \times d_1}$,  $b^{(2)} \in \mathbb{R}^{d_2}$,  $W^{(1)} \in \mathbb{R}^{d_1 \times d}$ and $b^{(1)} \in \mathbb{R}^{d_1}$.
To ensure the existence of the infinite width limit, we impose mild assumptions on the activation function \(\sigma\). The differentiability assumption in (i) is necessary to make sure that the PINN loss mathematically makes sense in a strong sense and that gradient type algorithms can be employed for training. The growth assumption in (ii) is needed when invoking central limit theorem and is satisfied by the RePU (i.e., the power of ReLU) and several Lipschitz activations, e.g., sigmoid, hyperbolic tangent and softplus function, and thus it is not restrictive.
\begin{assumption}\label{A:activation}
	Let \(m\) be the highest order of the linear differential operators.
	\begin{itemize}
		\item[(i)]  \(\sigma\) is differentiable
		up to the order \(m\), and \(\sigma^{(m)}\) is piecewise continuously differentiable.
		\item[(ii)] The derivatives of \(\sigma\) grows polynomially: there exist \(C\) and \(s>0\) such that for all \(0\leq k \leq m+1\), 
		\[ \sigma^{(k)}(x) \leq C(|x|^s + 1), \quad \text{a.e. in } \R. \]
	\end{itemize}
\end{assumption}

The following two results on the infinite width limit of two-layer NNs are partially covered by the counterparts for DNNs in Section \ref{S:DNTK-deep}. However, the explicit expressions \eqref{E:2-Cov} and \eqref{E:2-DNTK} play a crucial role in establishing the positivity of the DNTK for shallow NNs, and do not follow directly from the recursive formulas \eqref{E:deep-Cov} and \eqref{E:deep-DNTK}. We therefore present the results for shallow NNs separately. Throughout, the notation $a(z)$ and $b(z)$ are the zeroth order terms of the differential operators $\AA$ and $\BB$, respectively.
\begin{proposition}\label{P:2-ifyNN}
    Let the activation \(\sigma\) satisfy Assumption \ref{A:activation} and the parameters $\theta$ be initialized to i.i.d. standard Gaussian. The output of 
	the NN \( f_{\mu}^{(2)}:\PLDR_d\to \R \), for \( \mu \in [d_2] \), converges in 
	distribution to an i.i.d. centered Gaussian process \(f_\infty^{(2)}\) on \(\PLDR_d\) as $d_1\to\infty$:
	\[ \lim_{d_1\to\infty} f_{\mu}^{(2)}\overset{d}{\longrightarrow} f_\infty^{(2)}
		= \mathcal{GP}\left(0,\Sigma^{(2)}_{\infty}\right),
		\quad \mu \in [d_2]. \]
	The covariance function \(\Sigma^{(2)}_{\infty}:\PLDR_d\times\PLDR_d\to \R \) is defined by
	\begin{equation}\label{E:2-Cov}
		\Sigma^{(2)}_{\infty}\left(x^\AA,y^\BB\right) = \E_{u\sim \mathcal{N}(0,I), v\sim \mathcal{N}(0,1)}
		\left[ \mathcal{A}\sigma\left(u^\top x +  v\right)\mathcal{B}\sigma\left(u^\top y +  v\right) \right] 
	 	+ a(x)b(y),
	\end{equation} 
where the operators act on the input variable, e.g., $\mathcal{A}\sigma\left(u^\top x +  v\right) := 
	\mathcal{A}_z \left[\sigma\left(u^\top z +  v\right)\right] (x)$.
\end{proposition}

\begin{theorem}\label{T:2-ifyDNTK} 
Let the activation \(\sigma\) satisfy Assumption \ref{A:activation} and the parameters $\theta$ be initialized to i.i.d. standard Gaussian. Then the DNTK 
	\( \Theta^{(2)}_{\mu\nu}:\PLDR_d\times \PLDR_d\to \R \), for \(\mu,\nu \in [d_2]\), 
	converges in probability to a deterministic kernel \(\Theta^{(2)}_\infty\) on \(\PLDR_d\) as  \( d_1 \to \infty \):
	\[ \lim_{d_1\to\infty} \Theta^{(2)}_{\mu\nu} \overset{p}{\longrightarrow} 
		\Theta_{\infty}^{(2)}\delta_{\mu\nu}, \quad \mu,\nu \in [d_2]. \]
The kernel \(\Theta^{(2)}_{\infty}: \PLDR_d \times \PLDR_d \to \R\) is defined by
	\begin{equation}\label{E:2-DNTK}
	\begin{aligned}
		\Theta^{(2)}_{\infty}\left( x^\AA,y^\BB \right) 
		 =& \E_{u\sim \mathcal{N}(0,I), v\sim \mathcal{N}(0,1)}
		\left[ \mathcal{A}\sigma\left(u^\top x +  v\right)\mathcal{B}\sigma\left(u^\top y +  v\right) \right] 
			 + a(x)b(y) \\
		&  + \E_{u\sim \mathcal{N}(0,I), v\sim \mathcal{N}(0,1)}
		\left[ \sum_{k=1}^{d}\mathcal{A}\left(\sigma'\left(u^\top x +  v\right)x_k\right)
		\mathcal{B}\left(\sigma'\left(u^\top y +  v\right)y_k\right)  \right] \\
		&  + \E_{u\sim \mathcal{N}(0,I), v\sim \mathcal{N}(0,1)}
		\left[ \mathcal{A}\sigma'\left(u^\top x +  v\right)
		\mathcal{B}\sigma'\left(u^\top y +  v\right)  \right],
	\end{aligned}
	\end{equation}
	where $x_k$ and $y_k$ are the \(k\)-th component of the inputs $x$ and $y$ respectively. 
\end{theorem}

\begin{remark}
The four terms in the expression of \(\Theta^{(2)}_{\infty}\left( x^\AA,y^\BB \right)\)
 arise from the parameters $W^{(2)}$, $b^{(2)}$, $W^{(1)}$ and \(b^{(1)}\), respectively. The first two terms arising from parameters of the second layer are identical to that of \(\Sigma^{(2)}_\infty\). Except the second term, the remaining three terms of \(\Theta^{(2)}_{\infty}\) contribute to the positivity of the induced kernel matrix $\Theta_\infty^{(2)}$.
\end{remark}

\begin{remark}
The Gaussian initialization in Proposition \ref{P:2-ifyNN} and Theorem \ref{T:2-ifyDNTK} can be replaced by any mean-zero independent random variables with finite moments such that the entries in \(W^{(2)}, b^{(2)}, W^{(1)}\) and \(b^{(1)}\) have identical distributions respectively. Such initialization schemes include zero bias initialization \citep{Geifman:2020Similarity}, LeCun initialization \citep{Lecun:2002efficient} and He initialization \citep{He:2015delving}. 
\end{remark}

Next, we investigate the positivity of the matrices induced by the infinite width DNTK and the covariance function, cf. the formulas \eqref{E:2-DNTK} and \eqref{E:2-Cov}. 
Given a finite subset 
\( X = \{x_1^{\AA_1},\ldots, x_N^{\AA_N}\} \) in \(\PLDR_d\), such that the operators  
\(\AA_1,\ldots,\AA_N\in \LD_d \) with the zeroth order terms \(a_1(x),\ldots,a_N(x)\), we 
define the following two matrices
\[ \Theta^{(2)}_{\infty}(X):= \begin{bmatrix} \Theta^{(2)}_{\infty}
	(x_i^{\AA_i},x_j^{\AA_j}) \end{bmatrix}_{i,j\in[N]} \quad \mbox{and}\quad
	\qquad \Sigma^{(2)}_{\infty}(X):= \begin{bmatrix} \Sigma^{(2)}_{\infty}
	(x_i^{\AA_i},x_j^{\AA_j}) \end{bmatrix}_{i,j\in[N]}. \]
In practice, the points \(x_1,\ldots,x_N\) are pairwise distinct while
the operators \(\AA_1,\ldots,\AA_N \) are not necessarily distinct. First, from the expression 
\eqref{E:2-DNTK} of the matrix \(\Theta^{(2)}_{\infty}\left(x^\AA,y^\BB\right)\), we have
\begin{equation}\label{E:2-DNTK_D(x)}
	\Theta^{(2)}_{\infty}\left(x^\AA,y^\BB\right) = \E_{u\sim \mathcal{N}(0,I), v\sim \mathcal{N}(0,1)}
	\left[ D(x^{\AA};u,v)^\top D(y^{\BB};u,v) \right] + a(x)b(y),
\end{equation}
where \(D(x^{\AA};u,v)\) is a vector-valued function in \(u,v\) with \(d+2\) components given by 
\begin{equation}\label{E:2-D(x^AA)} 
	D\left(x^{\AA};u,v\right) =\left[\mathcal{A}\sigma\left(u^\top x +  v\right),
		\mathcal{A}(\sigma^{(1)}\left(u^\top x +  v\right)x_k),
		\mathcal{A}\sigma^{(1)}\left(u^\top x +  v\right) \right]^\top, \quad k\in[d].
\end{equation}  
Substituting the expression \eqref{E:2-DNTK_D(x)} into \(\Theta^{(2)}_{\infty}(X)\) gives
\begin{align*}
	\Theta^{(2)}_{\infty} \left(X\right) & = \begin{bmatrix} \E_{u,v} \left[ D(x_i^{\AA_i};u,v)^\top
	 D(x_j^{\AA_j};u,v) \right] + a_i(x_i)a_j(x_j) \end{bmatrix}_{i,j\in[N]} \\
	 & = \E_{u, v} \left[ D\left(X;u,v\right)^\top D\left(X;u,v\right) \right] 
	 + \left[a_1(x_1),\cdots , a_N(x_N)\right]^\top \left[a_1(x_1),\cdots , a_N(x_N)\right], 
\end{align*}
with the matrix \(D(X;u,v)\) given by 
\begin{equation}\label{E:2-D(X)} 
	D(X;u,v) = D\left(x_1^{\AA_1},\ldots, x_N^{\AA_N};u,v\right) 
	= \left[ D(x_1^{\AA_1};u,v), \cdots, D(x_N^{\AA_N};u,v) \right].
\end{equation}
Similarly, the matrix \(\Sigma^{(2)}_{\infty}(X)\) is given by
\[ \Sigma^{(2)}_{\infty}(X) = \E_{u, v} \left[ D_1\left(X;u,v\right)^\top D_1\left(X;u,v\right) \right] 
	 + \left[a_1(x_1),\cdots , a_N(x_N)\right]^\top \left[a_1(x_1),\cdots , a_N(x_N)\right], \]
with \(D_1\left(X;u,v\right)\) being the first row of the matrix \(D\left(X;u,v\right)\), given by
\[ D_1\left(X;u,v\right) = \left[\mathcal{A}\sigma\left(u^\top x_1 +  v\right),
\cdots, \mathcal{A} \sigma\left(u^\top x_N +  v\right) \right]. \]

The kernel matrix \(\Theta^{(2)}_{\infty} \left(X\right)\)
is always positive semi-definite:
\begin{align*}
	w ^\top \Theta^{(2)}_{\infty} \left(X\right) w 
	= \E_{u, v} \left[ \left\| D\left(X;u,v\right)w  \right\|^2 \right]
	+ \left\| \left[a_1(x_1),\cdots , a_N(x_N)\right]w \right\|^2 \geq 0,\quad \forall w \in\R^N.
\end{align*}
Note that the second term may vanish. Thus we aim to show the positivity 
using the first term. Note that the independent random variables \((u,v)\sim \mathcal{N}(0,I_{d+1})\) 
have full support on \(\R^{d+1}\) and the norm \(\left\| D\left(X;u,v\right)w \right\|\)
is continuous with respect to \((u,v)\) by the assumption on the activation \(\sigma\). Thus the expectation is strictly positive if 
\(D\left(X;u,v\right)w\) does not vanish everywhere. That is, the matrix 
\(\Theta^{(2)}_{\infty} \left(X\right)\) is positive definite if there is no vector \(w\in \R^N\) such that
\[	D\left(X;u,v\right) w \equiv 0 \in \R^{d+2}, \quad \forall  u\in\R^d, v\in\R. \] 
Therefore, each linearly independent row of \(D\left(X;u,v\right)\) contributes to 
the positivity to the kernel matrix $\Theta_\infty^{(2)}$. The discussion also holds for the covariance matrix
\(\Sigma^{(2)}_{\infty}(X)\), which only involves the first row \(D_1\left(X;u,v\right)\).
Below we will focus on the linear independence of the row $D_1\left(X;u,v\right)$, from which 
we can derive the positivity of both \(\Theta^{(2)}_{\infty}\) and \(\Sigma^{(2)}_{\infty}\). 
This boils down to the following equation for \(X = \{x_1^{\AA_1},\ldots, x_N^{\AA_N}\}\): 
\begin{equation}\label{E:LinRelation} 
	w_1\AA_1\sigma\left(u^\top x_1 + v\right) + \ldots + w_N\AA_N\sigma\left(u^\top x_N + v\right)
	\equiv 0, \quad  u\in\R^d, v\in\R. 
\end{equation}
We discuss the problem when the activation \(\sigma\) is smooth, RePU and general function separately.

\begin{remark}
The linear dependence for the other rows of \(D\left(X;u,v\right)\), arising from the components of \(D(x^{\AA};u,v)\), also takes the form of \eqref{E:LinRelation} except all 
    \(\AA\sigma  (u^\top x + v )\) changed to \(\AA\left(\sigma^{(1)}(u^\top x + v)x_k\right)\).
    For each \(k\in[d]\), we have 
    \begin{align}\label{E:AA-act*x} 
	   \AA \left( \sigma^{(1)} ( u^\top x + v ) x_k \right) 
	   &= x_k \AA \sigma^{(1)} ( u^\top x + v ) + \AA_{(k)}\sigma^{(1)} ( u^\top x + v )\\
       & = \left(x_k \AA + \AA_{(k)}\right) \sigma^{(1)} ( u^\top x + v ),\nonumber
    \end{align}
    with the operator \(\AA_{(k)}\in\LD_d\) obtained from \(\AA\) by the substitution
    \[ \partial^\alpha \mapsto
    \begin{cases}
    	\alpha_k \partial_{1}^{\alpha_{1}}\cdots \partial_{k}^{\alpha_{k}-1} 
    	\cdots\partial_{d}^{\alpha_{d}}, & \text{ if } \alpha_{k}\geq 1, \\
    	0, & \text{ if } \alpha_{k} = 0.
    \end{cases} \]
    Thus the corresponding equation can be treated similarly using the operator \(x_k \AA + \AA_{(k)}\).
\end{remark}

To identify conditions on the operator-marked points \(x_1^{\AA_1},\ldots,x_N^{\AA_N}\) and the 
activation \(\sigma\) so that for any nonzero \(w = (w_1,\ldots,w_N) \in\R^d\), equality 
\eqref{E:LinRelation} cannot happen, we first evaluate each term in equation \eqref{E:LinRelation}. 
For an element \(\AA\in \LD_d\), the notation \( I_\AA\) denotes the finite set of all multi-indices in \(\AA\) (i.e., the indices of the monomial derivatives with nonzero coefficients). Then any operator \(\AA \in \LD_d\) can be written as
\begin{equation}\label{E:LinDif}
	\AA = \sum_{\alpha \in  I_\AA} a_\alpha(z) \partial^{\alpha},
    \quad a_\alpha(z)\in \FF_d.
\end{equation}
Then, we have 
\( \AA\sigma(u^\top x + v) = \sum_{\alpha \in I_\AA} a_\alpha(x)
		\sigma^{(|\alpha|)}(u^\top x + v)u^{\alpha} \).
This motivates the following definition about the structure of the form (standard form).
\begin{definition} \label{D:sigma-poly}
	Given a sufficiently smooth function \(\sigma:\R\to \R\) and a point \(x\in \R^d\),  
	the expression for \(u\in\R^d\) and \(v\in\R\)
	\begin{equation}\label{E:sig-poly} 
		P = P(v + u^\top x ) = \sum_{k=1}^{K} a_k \sigma^{(m_k)}(v + u^\top x)u^{\alpha_k}
	\end{equation}
	is called a \(\sigma(v + u^\top x)\)\textbf{-polynomial in} \(u\) (\(\sigma\)\textbf{-polynomial} in short), where \(a_k\in\R\)$, $  \(m_k\in\Z_{\geq 0}\)
	and \(\sigma^{(m_k)}\) denotes the \(m_k\)-times derivatives of \(\sigma\) and \(\alpha_k\) are 
	indices in \(\R^d\). The terms $\sigma^{(m_k)}(v + u^\top x)u^{\alpha_k}$, $k=1,\ldots,K$, in the summation 
    are pairwise distinct. Two terms \(\sigma^{(m_i)}(v + u^\top x)u^{\alpha_i}\) and
	\(\sigma^{(m_j)}(v + u^\top x)u^{\alpha_j}\) are called \textbf{like terms} if \(u^{\alpha_i} = u^{\alpha_j}\) (i.e., \(\alpha_i=\alpha_j\)).
	Moreover,
	\begin{itemize}
		\item[(i)] $P$ is \textbf{non-degenerate}, if \(K\geq 1\)
		and all \(a_k\) are nonzero;
		\item[(ii)] $P$ is \textbf{orderly}, if \(m_k\) and \(\alpha_k\) satisfy
		\( m_1 - |\alpha_1| = m_2 - |\alpha_2| =\cdots= m_K - |\alpha_K|; \)
		\item[(iii)] $P$ is \textbf{reduced}, if the monomials \(u^{\alpha_1},\ldots, u^{\alpha_K}\) 
		have no common divisor, i.e.,
		\[ \mathrm{gcd}\left(u^{\alpha_1},\ldots, u^{\alpha_K}\right)=1. \]
	\end{itemize}
	If the \(\sigma(v + u^\top x)\)-polynomial \(P\) satisfies (ii) 
	and moreover (after combining the like terms and discarding the zero terms) satisfies condition (i), then it is said to be in the \textbf{standard form}.
\end{definition}

\begin{remark}
The non-degeneracy condition (i) only requires \(a_k,k\in [K]\) are nonzero, but the expression \eqref{E:sig-poly} may vanish for some particular functions \(\sigma\), e.g., polynomials. The orderly property implies
\[ \sigma^{(m_i)}(v + u^\top x)u^{\alpha_i} = \sigma^{(m_j)}(v + u^\top x)u^{\alpha_j}
\quad \Longleftrightarrow \quad u^{\alpha_i} = u^{\alpha_j}. \]
Thus for two like terms, their derivative orders of \(\sigma\) coincide, and the terms of the form $ (a_i\sigma^{(m_i)}(v + u^\top x) + a_j\sigma^{(m_j)}(v + u^\top x))u^{\alpha}$
for some index \(\alpha\) and \(m_i\neq m_j\) do not appear after combining like terms. Moreover, any standard 
	\(\sigma(v + u^\top x)\)-polynomial in \(u\) can be written as 
	\[ P = \sum_{k=1}^{K} a_k \sigma^{(m_k)}(v + u^\top x)u^{\alpha_k}, \]
	where the monomials \(u^{\alpha_1}, \ldots, u^{\alpha_K}\) are pairwise distinct and all \(a_k\)
    are nonzero.
\end{remark}

Note that each term in \eqref{E:LinRelation} is naturally an orderly \(\sigma\)-polynomial. Without loss of generality, we may always assume that all the \(w_1,\ldots,w_N\) are nonzero by discarding the zero ones.

\subsection{Smooth activation} \label{SubSec:LI-smo}
Now, we deal with equation \eqref{E:LinRelation} for a smooth activation \(\sigma\).

\begin{theorem}\label{T:LI-smooth_uv} 
Let the vectors \(\{x_i\}_{i=1}^N\subset \R^d\) be pairwise distinct. 
Let the activation \(\sigma:\R\to\R\) be smooth and that \(P_k\) are \(\sigma(v + u^\top x_k)\)-polynomials in \(u\) for \(k\in[N]\), with \(P_1,\ldots,P_N\) in the standard form. If 
	\[ w_1P_1(v + u^\top x_1) + \cdots + w_NP_N(v + u^\top x_N) \equiv 0,
		\quad \forall  u\in\R^d, v\in \R, \quad \mbox{with } \{w_i\}_{i=1}^N \subset \mathbb{R}\setminus\{0\}, \]
then \(\sigma\) is a polynomial.
\end{theorem}

In Theorem \ref{T:LI-smooth_uv}, there is an extra variable \(v\) in the argument of \(\sigma\) arising from the bias term \(b^{(1)}\) in the first layer of the NN. The variable \(v\) plays an important role at two steps in the proof. The first one is in the case \(N=1\).
We simplify the \(\sigma\)-polynomial \(P\) by setting some \(u_j,j\in[d]\) to zero, and finally 
get an identically zero \(\sigma\)-monomial, from which we deduce that \(\sigma\) is a polynomial. 
However, if the variable \(v\) does not exist, we may end up with a trivial equation.
The second one lies in the step of coordinate transformation. The presence of \(v\) makes the matrix in the linear transform simple to construct: one standard \(\sigma\)-polynomial
is easily changed to another standard \(\sigma\)-polynomial after affine transformation.

In the absence of the bias term, we require the points \(\{x_1,\ldots,x_N\}\) to be pairwise non-parallel and arrive at the homogeneous \emph{Cauchy-Euler equation} of order \(n\) in the case \(N=1\) (with $a_n\neq 0$):
\[ a_{n}x^{n}\sigma^{(n)}(x)+a_{n-1}x^{n-1}\sigma^{(n-1)}(x)+\dots +a_{0}\sigma(x)=0,
\quad a_i\in\R, i=0,1,\ldots,n. \]
Using the substitution \(t=\ln x\), we can find the general solution. Specifically, for \(x\in\R_+\) (the case \(x<0\) can be analyzed similarly), let \(x = e^t\) and \(y:=\sigma(x)\). Then, each term \(x^i \sigma^{(i)}(x)\) becomes \(D_t(D_t-1)\cdots (D_t-i+1)y \) under this transformation with \(D_t := \frac{\rm d}{{\rm d}t}\) and
\(y=y(t)\) satisfies the following linear ODE (with constant coefficients):
\[  L_t [y] = \sum_{i=0}^{n}a_i D_t(D_t - 1) \cdots  (D_t-i+1)y = 0. \]
Thus, the solutions can be classified by the roots of the characteristic polynomial 
$$p(t)=\sum_{i=0}^{n}a_i t(t - 1)\cdots (t-i+1). $$
For a real root \(\alpha\) with a multiplicity \(r\) to \(p(t)\), the \(r\) linearly independent solutions are given by
\[ \left\{ x^\alpha, x^\alpha\ln x, \cdots, x^\alpha (\ln x)^{r-1} \right\}; \]
while for the complex conjugate pair \(\alpha\pm i\beta\) with a multiplicity \(r\), there are  \(2r\) linearly independent solutions:
\[ 
\left\{ (\ln x)^{k}x^\alpha \cos(\beta\ln x),\; (\ln x)^{k}x^\alpha \sin(\beta\ln x): k=0,1,\ldots,r-1 \right\}. \]

\begin{theorem}\label{T:LI-smo_u} 
	Let the vectors  \(\{x_i\}_{i=1}^N\subset\R^d\) be nonzero and pairwise non-parallel. 
	Suppose that the activation \(\sigma:\R\to\R\) is smooth 
	and that \(P_k\) are \(\sigma(u^\top x_k)\)-polynomials in \(u\) for \(k\in[N]\),
	with \(P_1,\ldots,P_N\) in the standard form. If 
		\[ w_1P_1(u^\top x_1) + \cdots + w_NP_N(u^\top x_N) \equiv 0,\quad \forall  u\in\R^d, \quad \mbox{for } \{w_i\}_{i=1}^N\subset \mathbb{R}\setminus\{0\},\]
	then \(\sigma\) satisfies some homogeneous Cauchy-Euler equation.
\end{theorem}

These two theorems establish that, under the mild assumption that the input samples are either pairwise distinct or non-parallel (depending on whether a bias term is present), equality \eqref{E:LinRelation} cannot hold if either:
(i) the activation \(\sigma\) is not a polynomial with the presence of the bias term; or \(\sigma\) is not a solution to the homogeneous Cauchy–Euler equation in the absence of a bias term.
Importantly, this assumption is very mild and satisfied by most commonly used smooth activations, e.g., sigmoid, hyperbolic tangent, and softplus functions.
These results further indicate that the inclusion of a bias term enhances the likelihood of the kernel being positive definite. 

\subsection{General case}
In the proof of Theorem \ref{T:LI-smooth_uv}, \(\sigma\) is smooth so that differentiating the equation \eqref{E:Lin_Rel-sig_P} reduces it to the case \(N=1\).
Now, we relax the smoothness assumption using 
finite difference for the reduction process. The 
difference operator \(\Delta_{p}: \FF_d  \to \FF_d\) for an increment \(p=(p_1,\ldots,p_d)\in\R^d\) is defined by
\begin{gather*}
	\Delta_{p} f := f(x+p)-f(x) = f(x_1+p_1,\ldots,x_d+p_d)-f(x_1,\ldots,x_d).
\end{gather*}
We only analyze the problem 
with a bias term. 

\begin{theorem}\label{T:LI-general_uv} 
	Let the vectors \(\{x_i\}_{i=1}^N\subset \R^d\) be pairwise distinct. 
	Suppose that the activation \(\sigma:\R\to\R\) satisfies Assumption
	\ref{A:activation} (i) and that \(P_k\) are \(\sigma(v + u^\top x_k)\)-polynomials 
	in \(u\) for \(k\in[N]\), with \(P_1,\ldots,P_N\) in the standard form. If
		\[ w_1P_1(v + u^\top x_1) + \cdots + w_NP_N(v + u^\top x_N) \equiv 0, 
		\quad \forall u\in\R^d, v\in \R, \quad \mbox{for } \{w_i\}_{i=1}^N\subset \mathbb{R}\setminus\{0\}, \]
	then \(P_1\) satisfies some non-degenerate difference equation for \(u\). 
	In particular, if \(P_1 = \sigma(v + u^\top x_1)\), then \(\sigma\) is a polynomial. 
\end{theorem}

Note that there is also a version of Theorem \ref{T:LI-general_uv} without a bias term, which can be obtained by replacing the condition of pairwise distinct vectors with that of pairwise non‑parallel vectors, while keeping all the others unchanged. However, similar to the smooth case, when the \(\sigma\)-polynomial \(P_1\) satisfies the same difference equation, there should be stricter constraints on the activation \(\sigma\) in the presence of the bias term. 

\subsection{RePU function}
Now, we discuss equation \eqref{E:LinRelation} for the RePU without the bias term.
The case with the bias term follows similarly using the extended points 
\(y_i:=[1,x_i^\top]\in\R^{d+1}\) for \(i\in[N]\) and they 
are pairwise non-parallel if $\{x_i\}_{i=1}^N$ are pairwise distinct. 
The proof method can be found in \cite{Du:2019shallow,Gao:2023GD}.

\begin{theorem}\label{T:LI-RePU_u} 
Let the vectors \(\{x_i\}_{i=1}^N\subset\R^d\) be pairwise non-parallel. 
Suppose that the activation \(\sigma = z^q \mathbb{I}(z>0)\) and that \(P_k\) are standard \(\sigma(u^\top x_k)\)-polynomials in \(u\) for \(k\in[N]\), with the integer \(q\) greater than or equal to \(m\). If
		\[ w_1P_1(u^\top x_1) + \cdots + w_NP_N(u^\top x_N) = 0,\quad a.e.~ u\in\R^d, \]
	for some $\{w_i\}_{i=1}^N\subset \R$, then all the \(w_i\) vanish.
\end{theorem}

\section{DNTK for DNNs}\label{S:DNTK-deep}

In this section, we study the infinite width limit of the output and the DNTK for DNNs  and then derive the positivity of the covariance function and the infinite DNTK. The technical proofs are given in Appendix \ref{S-A:L-layer NN}.

First we give the notation of two formulas for computing high-order derivatives of multivariate functions.
The first one is the \textit{general Leibniz rule}, which extends the product rule in calculus. 
Let \(z_1, \ldots, z_d\) be the coordinate variables. For any multi-index \(\alpha\) and 
two functions \(f,g\) in \(z\), we have
\begin{equation}\label{E:Leibniz}
	\partial^{\alpha}(fg) = \sum_{\beta \leq \alpha}\binom{\alpha}{\beta}
	\partial^{\beta} f \partial^{\alpha-\beta} g.
\end{equation}
The next one is \textit{Fa\`a di Bruno's formula}, which is a generalization of the chain rule. 
A \emph{partition} of a set \(X\) is a set of the disjoint non-empty subsets of \(X\) whose union is \(X\).
Let \(x_1, \ldots, x_m\) be any \(m\) coordinate variables with each \(x_i\in \{ z_1, \ldots, z_d \}\), 
which may be not pairwise different. Let \(g = g(z_1,\ldots,z_d):\R^d \to \R\) 
be a multivariate function and \(f:\R\to \R\) a univariate function. Then we have
\begin{equation}\label{E:Faa di}
	\frac{\partial^m f(g)}{\partial x_1 \cdots \partial x_m}  = \sum_{\pi \in \Pi} f^{(|\pi|)}(g) 
	\cdot \prod_{A \in \pi} \frac{\partial^{|A|} g}{\prod_{j \in A} \partial x_j},
\end{equation}
where the index \(\pi\) runs through the set \(\Pi\) of all partitions of the set \([m]\)
and \(|\pi|\) denotes the number of block in the partition $\pi$,
the index \(A\) runs through all the blocks in the partition \(\pi\) 
and \(|A|\) denotes the cardinality of $A$, i.e., the number of elements in the block \(A\). For a multi-index 
\(\alpha = (\alpha_1, \ldots, \alpha_d)\), the notation \(S_\alpha\) denotes the set
consisting of \(i\) with a multiplicity \(\alpha_i\) for all \(i\in[d]\), namely
\[ S_\alpha := \left\{ 1_{1},\ldots,1_{\alpha_1},2_{1},\ldots,2_{\alpha_2},\cdots,
	d_{1},\ldots,d_{\alpha_d} \right\}. \]
Then each subset \(A\) of \(S_{\alpha}\) corresponds to a multi-index with the \(i\)-th position
given by the number of \(i\) contained in $A$. The corresponding multi-index is also denoted by \(A\).
Let \(\Pi_\alpha\) denote the set of all 
partitions of the set \(S_\alpha\). Define the variables 
\(\{x_1,\ldots,x_{|\alpha|}\}\) with the first \(\alpha_1\) ones corresponding to the coordinate variable 
\(z_1\), and then the next \(\alpha_2\) ones for the \(z_2\) and so on, while the last \(\alpha_d\) ones
for the variable \(z_d\). Then the Fa\`a di Bruno's formula reads
\begin{equation}\label{E:Faa-alpha}
	\partial^\alpha f(g) = \frac{\partial^{|\alpha|} f(g)}{\partial x_1 \cdots \partial x_{|\alpha|}}  
	= \sum_{\pi \in \Pi_\alpha} f^{(|\pi|)}(g) 
	 \prod_{A \in \pi} \frac{\partial^{|A|} g}{\prod_{j \in A} \partial x_j}
	=: \sum_{\pi \in \Pi_\alpha} f^{(|\pi|)}(g) \prod_{A \in \pi} \partial^A g.
\end{equation}
Here, we use the combinatorial notation to describe the Fa\`a di Bruno's formula since it is more concise. See \cite{ConstantineSavits:1996} for the expression using completely multi-index notation. 

Next, we give the infinite width limit of the DNN. 
For a linear operator \(\AA\in\LD_d\), we use \(D\AA\) to denote the vector of \(d\) dimension 
consisting of coefficients of \(\AA\) corresponding to the first order indices 
\(\partial_{z_1},\ldots,\partial_{z_d}\). Specifically, let \(a_i(z)\) be the 
the coefficient of the index \(\partial_{z_i}\) in \(\AA\), which may be zero, for \(i\in[d]\). Then \(D\AA = [a_1(z),\ldots,a_d(z)]^\top\). 

\begin{proposition}\label{P:deep-ifyNN}
	Suppose that the NN is given by \eqref{E:L-layerNN} with the activation \(\sigma\) satisfying 
	Assumption \ref{A:activation} and that all the parameters are initialized to i.i.d. standard Gaussian. 
	In the case of infinite width, i.e., \( d_1,\ldots, d_{L-1} \to \infty \) sequentially, the output of 
	the DNN \( f_{\mu}^{(L)}:\PLDR_d \to \R \), for \( \mu \in [d_L] \), converges in 
	distribution to an i.i.d. centered Gaussian process \(f_\infty^{(L)}\) on \(\PLDR_d\):
	\[ \lim_{d_1,\ldots d_{L-1}\to\infty} f_{\mu}^{(L)}\overset{d}{\longrightarrow} f_\infty^{(L)}
		= \mathcal{GP}\left(0,\Sigma^{(L)}_{\infty}\right),
		\quad \mu \in [d_L]. \]
	The covariance function \(\Sigma^{(L)}_{\infty}:\PLDR_d \times\PLDR_d \to \R \) 
	is defined recursively for \(\ell\in [L-1]\) by
	\begin{equation}\label{E:deep-Cov}
		\begin{aligned}
			\Sigma^{(1)}_{\infty}\left( x^\AA,y^\BB \right) 
			& = \left(a(x)x+D\AA(x)\right)^\top \left(b(y)y+D\BB(y)\right) + \gamma^2 a(x)b(y), \\
			\Sigma^{(\ell+1)}_{\infty}\left( x^\AA,y^\BB \right) 
			& = \E_{f^{(\ell)}_\infty } \left[ \left(\sum_{\alpha \in  I_\AA } a_\alpha (x) \sum_{\pi \in \Pi_\alpha} 
				\sigma^{(|\pi|)}\left( f^{(\ell)}_\infty  \left( x \right) \right) 
				\prod_{A \in \pi} f^{(\ell)}_\infty \left(x^{\partial^A}\right) \right) \cdot \right. \\
			& \qquad \left. \left( \sum_{\beta \in  I_\BB } b_\beta (y) \sum_{\lambda \in \Pi_\beta} 
				\sigma^{(|\lambda|)}\left( f^{(\ell)}_\infty  \left( y \right) \right) 
				\prod_{B \in \lambda} f^{(\ell)}_\infty \left(y^{\partial^B}\right) \right) \right]
			 	+ \gamma^2 a(x)b(y),
		\end{aligned}
	\end{equation}
	where \(\AA = \sum_{\alpha \in  I_\AA } a_\alpha (z) \partial^{\alpha}\) and \(
	\BB = \sum_{\beta \in  I_\BB } b_\beta (z) \partial^{\beta}\) with zeroth order terms \(a(z)\) and \(b(z)\). 
\end{proposition}

\begin{theorem}\label{T:deep-ifyDNTK}
Suppose that the DNN is given by \eqref{E:L-layerNN} with the activation \(\sigma\) satisfying 
Assumption \ref{A:activation} and that all the parameters are initialized to i.i.d. standard Gaussian.  In the case of infinite width, i.e., \( d_1,\ldots,d_{L-1} \to \infty \) sequentially, the DNTK 
	\( \Theta^{(L)}_{\mu\nu}:\PLDR_d \times \PLDR_d \to \R \), for \( \mu,\nu \in [d_L] \), 
converge in probability to a deterministic kernel:
	\[ \lim_{d_1,\ldots,d_{L-1}\to\infty} \Theta^{(L)}_{\mu\nu} \overset{p}{\longrightarrow} 
		\Theta_{\infty}^{(L)}\delta_{\mu\nu}, \quad \mu,\nu \in [d_L]. \]
	The scalar kernel \(\Theta^{(L)}_{\infty}: \PLDR_d  \times \PLDR_d  \to \R\) is defined 
	recursively for \(\ell\in [L-1]\) by
	\begin{equation}\label{E:deep-DNTK}
		\begin{aligned}
			\Theta^{(1)}_{\infty}\left( x^\AA,y^\BB \right) 
				& = \Sigma^{(1)}_{\infty}\left( x^\AA,y^\BB \right), \\
			\Theta^{(\ell+1)}_{\infty}\left( x^\AA,y^\BB \right)
				& = \Sigma^{(\ell+1)}_{\infty}\left( x^\AA,y^\BB \right) 
				    + \sum_{\alpha\in I_\AA }\sum_{\beta\in I_\BB } a_\alpha(x)b_\beta(y)\cdot \\
				& \qquad \sum_{\xi\leq\alpha}\sum_{\zeta\leq\beta}\binom{\alpha}{\xi}\binom{\beta}{\zeta}
					\dot{\Sigma}^{(\ell+1)}_{\infty}\left( x^{\partial^\xi},y^{\partial^\zeta} \right)
					\Theta^{(\ell)}_{\infty}\left( x^{\partial^{(\alpha-\xi)}},y^{\partial^{(\beta-\zeta)}} \right),
		\end{aligned}	
	\end{equation}
	where \(\AA = \sum_{\alpha \in  I_\AA } a_\alpha (z) \partial^{\alpha},
	\BB = \sum_{\beta \in  I_\BB } b_\beta (z) \partial^{\beta}\), and 
	\begin{equation}\label{E:dotCov}
		\dot{\Sigma}^{(\ell+1)}_{\infty} \left( x^\AA,y^\BB \right) 
		= \E_{f^{(\ell)}_\infty} \left[ \AA \sigma' \left( f^{(\ell)}_{\infty}(x) \right)
		\BB \sigma' \left( f^{(\ell)}_{\infty}(y) \right) \right],
	\end{equation}
	which has the same form as the expectation term in \(\Sigma^{(\ell+1)}_{\infty}\) by replacing the nonlinearity \(\sigma\) with \(\sigma'\). 
\end{theorem}

\begin{remark}
\cite{Jacot:2018NTK} proved that the recursive formula for the covariance  function of the Gaussian process, i.e., the infinite width output of DNNs, is given by
    \begin{equation*} 
    \begin{aligned}
    	\Sigma^{(1)}(x,y)  = x^\top y + \gamma^2\quad\mbox{and} \quad
    	\Sigma^{(\ell+1)}(x,y)  = \mathbb{E}_{f \sim \mathcal{GP}(0, \Sigma^{(\ell)})} 
    	\left[ \sigma(f(x)) \sigma(f(y)) \right] + \gamma^2,
    \end{aligned}
    \end{equation*} 
    and the formula for the NTK is 
    \begin{equation*} 
    \begin{aligned}
    	\Theta_{\infty}^{(1)}(x,y) = \Sigma^{(1)}(x,y)\quad\mbox{and} \quad
    	\Theta_{\infty}^{(\ell+1)}(x,y) = \Theta_{\infty}^{(\ell)}(x,y) 
    	\dot{\Sigma}^{(\ell+1)}(x,y) + \Sigma^{(\ell+1)}(x,y),
    \end{aligned}
    \end{equation*} 
    with the kernel \(\dot{\Sigma}^{(\ell+1)}\) given by
    \(\dot{\Sigma}^{(\ell+1)}(x,y) = \mathbb{E}_{f \sim \mathcal{GP}\left(0, \Sigma^{(\ell)}\right)} 
    \left[ \sigma'\left(f(x)\right) \sigma'\left(f(y)\right) \right] \).
Proposition \ref{P:deep-ifyNN} and Theorem \ref{T:deep-ifyDNTK} extend these findings to the DNTK. Note that the recursive formulas are similar for the regression and PINN cases. With both $\AA$ and $\BB$ being the identity operator, the recursion formulas \eqref{E:deep-Cov} and \eqref{E:deep-DNTK} recover these existing results. Moreover, it can be verified that equations \eqref{E:deep-Cov} and \eqref{E:deep-DNTK} recover \eqref{E:2-Cov}
	and \eqref{E:2-DNTK} respectively when putting \(\ell=1\).
\end{remark}

Now we prove the positivity of \(\Sigma^{(L)}_\infty\) and \(\Theta^{(L)}_\infty\) in the presence of a bias term, i.e., \(\gamma>0\). Note that the expectation in the expression of \(\Sigma^{(\ell+1)}_{\infty}\left( x^\AA,y^\BB \right)\) \eqref{E:deep-Cov} is taken with respect to the Gaussians from the \(\ell\)th layer output \(f^{(\ell)}_\infty\), all of which are of the form \(f^{(\ell)}_\infty(x), ~f^{(\ell)}_\infty(x^{\partial^A}), ~f^{(\ell)}_\infty(y), ~f^{(\ell)}_\infty(y^{\partial^B})\). By Proposition \ref{P:deep-ifyNN}, these Gaussian variables constitute a Gaussian vector. However, the components may not be independent, i.e., the corresponding covariance matrix may be degenerate. For the subsequent analysis, we first identify a maximal independent subset whose covariance matrix has the same rank as the full set, which is called a \emph{basis} below.
The key strategy is to prove the following property LI (linear independence) by induction on the number of layers \(L\). 
\begin{definition}\label{D:Property (LI)}
The NN \(f^{(L)}\) \eqref{E:L-layerNN} is said to have \textbf{Property (LI)}, if for any set of pairwise distinct points in \(\PLDR_d\) of the form
\[ X = \left\{ x_1^{\partial^{\alpha_{11}}},\ldots,x_1^{\partial^{\alpha_{1K_1}}},\cdots, x_N^{\partial^{\alpha_{N1}}},\ldots,x_N^{\partial^{\alpha_{NK_N}}} \right\}, 
\]
where the Euclidean points \( \{x_1,\ldots,x_n\}\subset \R^d \) are pairwise distinct and the monomial indices \(\{\alpha_{i1},\ldots,\alpha_{iK_i}\}\) are pairwise different for each \(i\in[N]\), then the Gaussian vector  
	\[ f^{(L)}_\infty(X) = \left[ f^{(L)}_\infty\left(x_1^{\partial^{\alpha_{11}}}\right),\ldots,
	f^{(L)}_\infty\left(x_1^{\partial^{\alpha_{1K_1}}}\right),\cdots, 
	f^{(L)}_\infty\left(x_N^{\partial^{\alpha_{N1}}}\right),\ldots,
	f^{(L)}_\infty\left(x_N^{\partial^{\alpha_{NK_N}}}\right) \right]^\top \]
	is non-degenerate, i.e., the covariance matrix \(\Cov(X) = \Sigma^{(L)}_\infty(X)\) has full rank.
\end{definition}

This property helps to find a basis among Gaussians in the expectation term in \eqref{E:deep-Cov} which will be proved by mathematical induction on the depth of DNNs. 
For \(\gamma=0\), i.e., without bias term, the assumption on the sampling points \( \{x_1,\ldots,x_n\}\subset \R^d \) should be modified 
as pairwise non-parallel to make sure that the shallow NNs have Property (LI). 
Similar to Remark \ref{R:non-smooth}, if \(\sigma\) is not smooth but has differentiability up to some order, say \(m+1\), one should restrict the maximum order of the monomial indices, i.e., replace \(\MM_d\) by \(\MM_d^{(m)}\).

We use the case \(L=2\) to illustrate Definition \ref{D:Property (LI)}.
Given a set of points \(X\), by the expression \eqref{E:2-Cov} of \(\Sigma^{(2)}_\infty\), for all \(w=[w_1,\ldots,w_N]\in\R^{K_1+\ldots+K_N}\), we have 
\begin{align*}
	w^\top \Sigma^{(2)}_\infty(X) w = \E_{u,v} 
	\left[ \left\| D_1\left(X;u,v\right)w  \right\|^2 \right]
	+ \left\| \gamma \left[a_{11},\ldots,a_{1K_1},\cdots,a_{N1},\ldots,a_{NK_N}\right]w \right\|^2 \geq 0,
\end{align*}
where \(u\sim \NN(0,I_d), v\sim \NN(0,1)\) and \(D_1\left(X;u,v\right)\) is a vector function of \((u,v)\) given by 
\[ \left[ \partial^{\alpha_{11}}\sigma(u^\top x_1 + \gamma v),\ldots,
 \partial^{\alpha_{1K_1}}\sigma(u^\top x_1 + \gamma v),\cdots,
 \partial^{\alpha_{N1}}\sigma(u^\top x_N + \gamma v),\ldots,
 \partial^{\alpha_{NK_N}}\sigma(u^\top x_N + \gamma v) \right], \]
and \(a_{ik}\) for \(i\in[N]\) and \(1\leq k\leq K_i\) is the zeroth order term of 
\(\partial^{\alpha_{ik}}\) at the point \(x_i\), i.e.,
\[ a_{ik} = \begin{cases}
	1, & \text{ if } |\alpha_{ik}| = 0, \\
	0, & \text{ if } |\alpha_{ik}| > 0.
\end{cases} \]
In order to prove \(w^\top \Sigma^{(2)}_\infty(X) w >0\), it suffices to prove that the function
\(D_1\left(X;u,v\right)w\) is not identically zero for any \(w\neq 0\) since it is continuous and 
\((u,v)\) has full support over \(\R^{d+1}\). Note that
\begin{equation}\label{E:Property (LI)}
    D_1\left(X;u,v\right)w = \sum_{i=1}^{N} \sum_{k=1}^{K_i}w_{ik} 
    \sigma^{(\alpha_{ik})}(u^\top x_i + \gamma v)u^{\alpha_{ik}}.
\end{equation}
For each \(i\in [N]\), the term \(\sum_{k=1}^{K_i}w_{ik} 
\sigma^{(\alpha_{ik})}(u^\top x_i + \gamma v)u^{\alpha_{ik}}\) is an orderly 
\(\sigma(u^\top x_i + \gamma v)\)-polynomial in \(u\), and is non-degenerate if not 
all \(w_{ik}\) are zero for \(1\leq k\leq K_i\). By Theorems \ref{T:LI-smooth_uv} and \ref{T:LI-RePU_u}, if \(w\neq 0\), the linear sum \(D_1\left(X;u,v\right)w\) is not identically zero under the assumption therein. Therefore, if \(\sigma\) is RePU or non-polynomial smooth function, then the two-layer NN \(f^{(2)}\) with \(\gamma>0\) has Property (LI). The next theorem used to prove Theorem \ref{T:cor_Pos_DNTK} shows that DNNs will have Property (LI) once the shallow network has, and both proofs are given in  Appendix \ref{S-A:L-layer NN}.

\begin{theorem}\label{T:Property (LI)}
	Let the NN \(f^{(L)}\) be given by \eqref{E:L-layerNN} such that \(\gamma>0\) and the activation \(\sigma\) satisfy Assumption \ref{A:activation}. Suppose that \(f^{(2)}\) satisfies \textup{Property (LI)}, then 
	\(f^{(L)}\) has \textup{Property (LI)} for all \(L\geq 2\).
\end{theorem}

\section{Discussions} \label{S:Discussion}

In this work, we have investigated the positivity of differential neural tangent kernel (DNTK) for deep neural networks,
which is crucial to study the training behavior of PINNs in the kernel regime. 
For shallow neural networks, we obtain the positivity in Theorems \ref{T:LI-smooth_uv}
and \ref{T:LI-RePU_u} for smooth and RePU activation respectively; while for the general activation, we reduce the positivity problem into
a difference-differential equation \eqref{E:FD-final_step} by Theorem \ref{T:LI-general_uv}. 
The results in Section \ref{S:DNTK-shallow} imply that the presence of the bias term makes the DNTK more likely to be positive definite. 
One can then obtain the positivity of the DNTK for DNNs through Theorem \ref{T:Property (LI)}, while the proof indicates that increasing network depth enhances the positivity of the infinite DNTK.

In addition to the positivity, spectral properties of the DNTK that deeply affects the training dynamics of neural networks, e.g.,
smallest eigenvalue, largest eigenvalue \citep{Xu:2024NGD}, decay rate of the spectrum
of the induced integral operator, and the associated RKHS, 
still remain unknown. The domain of the DNTK, i.e., the Cartesian product \(\LD_d \times \R^d\), 
lacks an appropriate topological or analytic structure, which makes it intractable to study the spectral property in the multiple operators case. Moreover, the discussion is restricted to linear operators, while nonlinear differential equations are common in practice.

\section*{Acknowledgments} Bangti Jin is supported by Hong Kong RGC General Research Fund (Project 14306423 and Project 14306824), and a start-up fund from The Chinese University of Hong Kong. Longjun Wu is partly supported by International Joint PhD Supervision Scheme of The Chinese University of Hong Kong.

\appendix

\section{Proofs for section \ref{S:DNTK-shallow}} \label{S-A:DNTK-shallow} 

We frequently use the multi-index notation in the proofs. 
In $\mathbb{R}^d$, a multi-index is a \(d\)-tuple \(\alpha = (\alpha_1,\ldots, \alpha_d)\), 
with each \(\alpha_i\in \Z_{\geq 0}\) and the norm
\(|\alpha| := \alpha_1 + \ldots + \alpha_d\). The \(\alpha\)-th power for 
\(x\in \R^d\) and its orders are defined by
\[  x^{\alpha }=x_{1}^{\alpha_{1}}\cdots x_{d}^{\alpha_{d}}, \quad \text{ with } \;
\ord(x^\alpha) = |\alpha|, \; \ord_{x_i}(x^\alpha) = \alpha_i \text{ for } i\in[d], 
\]
respectively. 
Given two indices $\alpha = (\alpha_1,\ldots, \alpha_d)$ and \(\beta = (\beta_1,\ldots, \beta_d)\), the componentwise sum and difference are defined by 
\(\alpha \pm \beta =(\alpha_{1}\pm \beta_{1},\ldots,\alpha_{d}\pm \beta_{d})\).
The partial order is given by $\alpha \leq \beta$ if and only if $
\alpha _{i}\leq \beta _{i}$ for all $i\in \{1,\ldots ,d\}$. The \(\alpha\)th-partial derivative $\partial^\alpha$ is given by
$\partial^{\alpha} := \partial_{1}^{\alpha_{1}}\cdots \partial_{d}^{\alpha_{d}} 
= \frac{\partial^{|\alpha|}}{\partial_{x_1}^{~\alpha_1}\cdots\partial_{x_d}^{~\alpha_d}}.$

For the analysis of the NNs in the infinite width limit, we need the following multidimensional central limit theorem (MCLT) \cite[Example 2.18]{Van:2000Asymptotic}.

\begin{lemma}[Multidimensional central limit theorem] \label{lem:MCLT}
Let \( \{ X_1, X_2, X_3, \cdots \} \) be a sequence of i.i.d. random vectors on \(\R^d\) with 
	mean vector \( \mu \in \R^d \) and covariance matrix 
	\( \Sigma \in \R^{d\times d}\).
	Then, as \( n \) tends to infinity, the random vectors \( \sqrt{n}(\overline{X}_n - \mu) \)
	with \(\overline{X}_n = \frac{1}{n}\sum_{i=1}^{n}X_i\),
	converge in distribution to a normal vector \(\mathcal{N}_d(0, \Sigma)\): 
	\[ \lim_{n\to \infty} \sqrt{n} (\overline{X}_n - \mu) = \lim_{n\to \infty}\frac{X_1+X_2+\cdots+X_n - n\mu}{\sqrt{n}}
	\overset{d}{\longrightarrow} \mathcal{N}_d(0, \Sigma). \]
\end{lemma}

\subsection{Proof of Proposition \ref{P:2-ifyNN}}

	For any \(\mu \in [d_2]\), the value \(f_{\mu}^{(2)}\left(x^\AA\right)\), i.e., 
	the differential of the NN \(f_{\mu}^{(2)}(x)\) under the operator \(\AA\), is given by
	\begin{align*}
		\mathcal{A}f_{\mu}^{(2)}(x) & = \mathcal{A}\left( \frac{1}{\sqrt{d_1}} W^{(2)}_{\mu} 
		\sigma \left( W^{(1)}x +  b^{(1)} \right) +  b^{(2)}_{\mu}  \right) \\
		& = \frac{1}{\sqrt{d_1}} W^{(2)}_{\mu} \mathcal{A}
		\sigma \left( W^{(1)}x +  b^{(1)} \right) + \mathcal{A} b^{(2)}_{\mu} \\
		& = \frac{1}{\sqrt{d_1}} \sum_{i=1}^{d_1} W^{(2)}_{\mu,i}  
		\mathcal{A}\sigma \left( W^{(1)}_{i}x +  b^{(1)}_{i} \right) +  a(x)b^{(2)}_{\mu}.
	\end{align*}
	For any \(i\in[d_1]\) and \(x\in \R^d\), we first prove the mean square integrability
	\[ \E_{W^{(2)}_{\mu,i}, W^{(1)}_{i}, b^{(1)}_{i}}\left[ \left(W^{(2)}_{\mu,i}  
		\mathcal{A}\sigma \left( W^{(1)}_{i}x +  b^{(1)}_{i} \right)\right)^2 \right]<\infty. \]
	Since the operator \(\AA\) takes the form \(\sum a_\alpha(x)\partial^{\alpha}\), by the independence between $W^{(1)}$, $b^{(1)}$ and $W^{(2)}$ and the identity 
	\( \E_{v\sim \NN(0,1)} \left[ v^2 \right] = \Cov(v) =1\) for the standard Gaussian random variable \(v\), we deduce
    \begin{align*}
        \E_{W^{(2)}_{\mu,i}, W^{(1)}_{i}, b^{(1)}_{i}}\left[ \left(W^{(2)}_{\mu,i}  
		\mathcal{A}\sigma \left( W^{(1)}_{i}x +  b^{(1)}_{i} \right)\right)^2 \right]
        = \E_{W^{(1)}_{i}, b^{(1)}_{i}}\left[ \left(\sum a_\alpha(x)\partial^{\alpha}
		\sigma \left( W^{(1)}_{i}x +  b^{(1)}_{i} \right)\right)^2 \right].
    \end{align*}
    By the Cauchy-Schwarz inequality, it suffices to analyze the case
	$\mathcal{A} = \frac{\partial^m}{\partial x_{k_1} \cdots \partial x_{k_m}}$ for some $ m\in\Z_{>0}$ and  $ k_1,k_2,\ldots,k_m\in [d]$.
In view of the identity
	\begin{gather*}
		\frac{\partial^m \sigma ( W^{(1)}_{i}x +  b^{(1)}_{i} )}{
		\partial x_{k_1} \cdots \partial x_{k_m}} 
		= \sigma^{(m)} ( W^{(1)}_{i}x +  b^{(1)}_{i} )W^{(1)}_{i,k_1}\cdots W^{(1)}_{i,k_m},
	\end{gather*}
and the polynomial growth of the activation \(\sigma\) in Assumption \ref{A:activation}, we have
	\begin{align*}
		& \E_{W^{(1)}_{i}, b^{(1)}_{i}}\left[ \left( 
		\mathcal{A}\sigma \left( W^{(1)}_{i}x +  b^{(1)}_{i} \right)\right)^2 \right] \\
		 =& \E_{W^{(1)}_{i}, b^{(1)}_{i}} \left[ \left(
		\sigma^{(m)} \left( W^{(1)}_{i}x +  b^{(1)}_{i} \right)W^{(1)}_{i,k_1}\cdots W^{(1)}_{i,k_m}
		\right)^2 \right] \\
		 \lesssim &\E_{W^{(1)}_{i}, b^{(1)}_{i}} \left[
		 \left(\left| W^{(1)}_{i}x +  b^{(1)}_{i} \right|^s +1\right)^2 
		  \left(W^{(1)}_{i,k_1}\cdots W^{(1)}_{i,k_m}\right)^2 \right] < \infty.
	\end{align*}
	Given a set of marked points \(X = \{ x_1^{\AA_1},\ldots,x_N^{\AA_N} \}\subset \PLDR_d\), 
	since the parameters are independent, by Lemma \ref{lem:MCLT}, we have 
	\begin{equation}\label{E:app-MCLT} 
		\lim_{d_1\to\infty} \frac{1}{\sqrt{d_1}} \sum_{i=1}^{d_1} \left[ W^{(2)}_{\mu,i}  
		\mathcal{A}_j \sigma \left( W^{(1)}_{i}x_j +  b^{(1)}_{i} \right) \right]_{1\leq j\leq N}
		\overset{d}{\longrightarrow} \mathcal{N}_N(0,\Sigma),
	\end{equation}
	where the covariance $\Sigma(x^\AA,y^\BB)$ for any two elements \(x^\AA,y^\BB\in X\) is given by
	\begin{align*}
		\Sigma\left(x^\AA,y^\BB\right) & = \E_{W^{(2)}_{\mu,i}, W^{(1)}_{i}, b^{(1)}_{i}}
		\left[ W^{(2)}_{\mu,i}\mathcal{A}\sigma \left( W^{(1)}_{i}x +  b^{(1)}_{i} \right)
		W^{(2)}_{\mu,i}\mathcal{B}\sigma \left( W^{(1)}_{i}y +  b^{(1)}_{i} \right) \right]\\
		& = \E_{W^{(1)}_{i}, b^{(1)}_{i}}
		\left[ \mathcal{A}\sigma \left( W^{(1)}_{i}x +  b^{(1)}_{i} \right)
		\mathcal{B}\sigma \left( W^{(1)}_{i}y +  b^{(1)}_{i} \right) \right].
	\end{align*}
	Let \(a_j(x)\) be the zeroth order term of \(\AA_j\) for \(j\in [N]\). 
	Combining the identity \eqref{E:app-MCLT} with the bias term, we have
	\begin{align*}
		\lim_{d_1\to\infty} f_{\mu}^{(2)}(X) \overset{d}{\longrightarrow} 
		\mathcal{N}_N(0,\Sigma) + b^{(2)}_{\mu} \left[a_j(x_j)\right]_{1\leq j\leq N}
		= \mathcal{N}_N\left(0,\Sigma^{(2)}_{\infty}(X)\right).
	\end{align*} 
	Since \(X\) is arbitrarily chosen, we conclude that the limit \(\lim_{d_1\to\infty} f_{\mu}^{(2)}\) is a 
	centered Gaussian process on the product space \(\PLDR_d\) with the covariance function given by 
	\(\Sigma^{(2)}_{\infty}\).

Next we consider the independence of the \(d_2\) components \(\{ f_{1,\infty}^{(2)},\cdots,  f_{d_2,\infty}^{(2)}\}\) of the NN output. Given any \(d_2\) points \(S=(x_{1}^{\AA_{1}},\ldots, x_{d_2}^{\AA_{d_2}}) \) in the product space \(\PLDR_d\), the random vector 
\begin{align*}
  f^{(2)}_{*,\infty}(S) = \begin{bmatrix}
			f_{1,\infty}^{(2)}\big( x_{1}^{\AA_{1}} \big) & \cdots & 
			f_{d_2,\infty}^{(2)}\big( x_{d_2}^{\AA_{d_2}} \big)
	\end{bmatrix}^\top \in \R^{d_2}
	\end{align*}
	is the limit in distribution of a sequence of random vectors indexed by \(d_1\)
	\begin{align*}
        f^{(2)}_{*,d_1}(S) = 
		\begin{bmatrix}
			f_{1,d_1}^{(2)}\left( x_{1}^{\AA_{1}} \right) \\ \vdots \\ 
			f_{d_2,d_1}^{(2)}\left( x_{d_2}^{\AA_{d_2}} \right)
		\end{bmatrix} = \begin{bmatrix}
			\frac{1}{\sqrt{d_1}} \sum_{i=1}^{d_1} W^{(2)}_{1,i}  
		\mathcal{A}_1 \sigma \left( W^{(1)}_{i}x_1 +  b^{(1)}_{i} \right) \\ \vdots \\ 
		\frac{1}{\sqrt{d_1}} \sum_{i=1}^{d_1} W^{(2)}_{d_2,i}  \mathcal{A}_{d_2} 
		\sigma \left( W^{(1)}_{i}x_{d_2} +  b^{(1)}_{i} \right) 
		\end{bmatrix} + 
        \begin{bmatrix}
             a_1(x_1)b^{(2)}_{1}  \\ \vdots \\ a_{d_2}(x_{d_2})b^{(2)}_{d_2}
        \end{bmatrix}.
	\end{align*}
	By Lemma \ref{lem:MCLT}, the first term converges to a multivariate normal vector as \(d_1\to\infty\) with covariance 
	\begin{align*} 
    &\Cov (x_{p}^{\AA_{p}}, x_{q}^{\AA_{q}} )\\
	=& \E_{W^{(2)}_{p,i}, W^{(2)}_{q,i}, W^{(1)}_{i}, b^{(1)}_{i}}
    \left[ W^{(2)}_{p,i} \mathcal{A}_p \sigma \left( W^{(1)}_{i}x_p + b^{(1)}_{i} \right)
    W^{(2)}_{q,i} \mathcal{A}_q \sigma \left( W^{(1)}_{i}x_q + b^{(1)}_{i} \right)\right] = 0   
    \end{align*} 
	for \(p\neq q\in[d_2]\) by the independence of parameters. Meanwhile, the second term is a Gaussian vector with independent components.
	Therefore, the limit \(f^{(2)}_{*,\infty}(S)\) is the sum of two independent normal vectors with zero covariance, and thus it is a normal vector with zero covariance 
    and its components are independent. By the arbitrariness of \(S\), the \(d_2\) components \(\{ f_{1,\infty}^{(2)},\cdots,  f_{d_2,\infty}^{(2)}\}\) are independent.

\subsection{Proof of Theorem \ref{T:2-ifyDNTK}}
We denote by \( \theta_{\ell}\) with \(\ell = 1, 2\) the set of all \(\ell\)-layer parameters (i.e., the weight matrix \( W^{(\ell)} \) and bias vector \( b^{(\ell)} \)). 
	Given any \(\mu,\nu\in[d_2]\) and \(x,y\in\R^d\), we split the inner product into two parts
	\begin{align*}
		\Theta^{(2)}_{\mu\nu}\left( x^\AA,y^\BB \right)
		= \llan \nabla_\theta \mathcal{A}f^{(2)}_\mu(x;\theta), \nabla_\theta\mathcal{B}f^{(2)}_\nu(y;\theta)  \rran
		=: \textrm{I} + \textrm{II},
	\end{align*}
	with
	\begin{gather*}
		\textrm{I} = \llan \nabla_{\theta_{2}} \mathcal{A}f^{(2)}_\mu(x;\theta), \nabla_{\theta_2}\mathcal{B}f^{(2)}_\nu(y;\theta)  \rran \quad \text{ and } \quad
		\textrm{II} = \llan \nabla_{\theta_{1}} \mathcal{A}f^{(2)}_\mu(x;\theta), \nabla_{\theta_1}\mathcal{B}f^{(2)}_\nu(y;\theta)  \rran.
	\end{gather*}
	For the term \textrm{I}, from the expression
	\begin{align*}
		\mathcal{A}f^{(2)}_\mu(x;\theta) = \frac{1}{\sqrt{d_1}} \sum_{i=1}^{d_1} 
		W^{(2)}_{\mu,i} \AA \sigma \left( W^{(1)}_{i} x +  b^{(1)}_i \right) + \AA b^{(2)}_\mu,
	\end{align*}
	we have
	\begin{align*}
		\textrm{I} & =  \llan 
			\nabla_{\theta_{2}} \frac{1}{\sqrt{d_1}} W^{(2)}_{\mu} 
				\AA \sigma \left( W^{(1)}x + b^{(1)} \right) + \AA b^{(2)}_\mu , 
			\nabla_{\theta_{2}} \frac{1}{\sqrt{d_1}} W^{(2)}_{\nu} 
				\BB \sigma \left( W^{(1)}x + b^{(1)} \right) + \BB b^{(2)}_\nu \rran  \\
		& = \frac{1}{d_1} \sum_{k=1}^{d_2} \sum_{i=1}^{d_1}
		 	\delta_{k \mu} \AA \sigma \left( W^{(1)}_{i} x +  b^{(1)}_i \right)
			\delta_{k \nu} \BB \sigma \left( W^{(1)}_{i} x +  b^{(1)}_i \right)
			+ \delta_{\mu \nu} a(x)b(y) \\
		& = \delta_{\mu \nu} \frac{1}{d_1} \sum_{i=1}^{d_1}
		 	\AA \sigma \left( W^{(1)}_{i} x +  b^{(1)}_i \right)
			\BB \sigma \left( W^{(1)}_{i} x +  b^{(1)}_i \right)
			+ \delta_{\mu \nu} a(x)b(y).
	\end{align*} 
	For the term \textrm{II}, we first compute the derivative of \(\AA f^{(2)}\) with respect to each element \(\rho \in \theta_{1}\). 
	Since the activation \(\sigma\) is continuously differentiable, the operators 
    \(\partial_\rho\) and \(\AA\) are commutative:
	\begin{align*}
		\partial_\rho\left( \AA f^{(2)}_{\mu}(x;\theta)\right)
		& = \partial_\rho\left( \frac{1}{\sqrt{d_1}} W^{(2)}_{\mu} 
			\AA \sigma \left( W^{(1)}x + b^{(1)} \right) + \AA b^{(2)}_\mu \right) \\
		& = \frac{1}{\sqrt{d_1}} W^{(2)}_{\mu} \AA \partial_\rho \sigma \left( W^{(1)}x + b^{(1)} \right)  \\
		& = \frac{1}{\sqrt{d_1}} W^{(2)}_{\mu} \AA \left(\sigma' \left( W^{(1)}x + b^{(1)} \right)
			\odot \partial_\rho\left(W^{(1)}x + b^{(1)}\right)\right)\\
		& = \frac{1}{\sqrt{d_1}} \sum_{i=1}^{d_1} W^{(2)}_{\mu,i}
			\AA \left(\sigma' \left( W^{(1)}_{i}x + b^{(1)}_i  \right) 
			\partial_\rho\left(W^{(1)}_{i}x + b^{(1)}_i \right)\right), 
	\end{align*}
	where the notation \(\odot\) denotes the Hadamard product (i.e., componentwise product) between two vectors. Similarly, we have
	\begin{align*}
		\partial_\rho\left( \BB f^{(2)}_{\nu}(y;\theta)\right)
		= \frac{1}{\sqrt{d_1}} \sum_{j=1}^{d_1} W^{(2)}_{\nu,j}
			\BB \left(\sigma' \left( W^{(1)}_{j}y + b^{(1)}_j  \right) 
			\partial_\rho\left(W^{(1)}_{j}y + b^{(1)}_j \right)\right).
	\end{align*}
	Note that there is no common parameter between \(W^{(1)}_{i}\) and \(W^{(1)}_{j}\)
	if \(i\neq j\). Consequently 
	\begin{align*}
		\textrm{II} & = \sum_{\rho\in \theta_{1}} \partial_\rho\left( \AA f^{(2)}_{\mu}(x;\theta)\right)
		\partial_\rho\left( \BB f^{(2)}_{\nu}(y;\theta)\right) \\
		& = \frac{1}{d_1} \sum_{i,j=1}^{d_1}  W^{(2)}_{\mu,i} W^{(2)}_{\nu,j} \cdot 
			\sum_{\rho\in \theta_{1}} \AA \left( \sigma' \left( W^{(1)}_{i}x + b^{(1)}_i  \right) 
			\partial_\rho\left(W^{(1)}_{i}x + b^{(1)}_i \right) \right) \\
            & \hspace{5cm}
			\BB \left( \sigma' \left( W^{(1)}_{j}y + b^{(1)}_j  \right)
			\partial_\rho\left(W^{(1)}_{j}y + b^{(1)}_j \right) \right) \\
		& = \frac{1}{d_1} \sum_{i,j=1}^{d_1}  W^{(2)}_{\mu,i} W^{(2)}_{\nu,j} \cdot \delta_{ij} 
			\sum_{k=0}^{n} \AA \left( \sigma' \left( W^{(1)}_{i}x + b^{(1)}_i  \right) x_k \right)
			\BB \left( \sigma' \left( W^{(1)}_{j}y + b^{(1)}_j  \right)y_k \right)\\
		& = \frac{1}{d_1} \sum_{i=1}^{d_1} W^{(2)}_{\mu,i} W^{(2)}_{\nu,i}
			\sum_{k=0}^{n} \AA \left( \sigma' \left( W^{(1)}_{i}x + b^{(1)}_i  \right) x_k \right)
			\BB \left( \sigma' \left( W^{(1)}_{i}y + b^{(1)}_i  \right)y_k \right),
	\end{align*}
using the convention \(x_0=y_0\equiv 1\). The desired assertion follows from the law of large numbers.

\subsection{Proof of Theorem \ref{T:LI-smooth_uv}}

First, we prove that the standard property of a \(\sigma\)-polynomial is invariant under differentiation. 
This is key to proving linear independence 
of $\AA_i\sigma(u^\top x_i + v)$, $i\in[N]$ when \(\sigma\) is smooth.

\begin{lemma}\label{L:StandardForm}
	Let the activation \(\sigma:\R\to \R\) be smooth and the point \(x\in \R^d\) be fixed.
	If \(x_j\neq 0\) for some \(j\in [d]\) and 
	\[ P = \sum_{k=1}^{K} a_k \sigma^{(m_k)}(v + u^\top x)u^{\alpha_k} \]
	is a \(\sigma\)-polynomial in the standard form, then its derivative $\frac{\partial P}{\partial u_j}$ with respect to \(u_j\)
	is also a \(\sigma\)-polynomial in the standard form. Therefore, for any positive integer \(M\), $\frac{\partial^M P}{\partial u_j^{M}}$
remains a standard \(\sigma\)-polynomial. 
\end{lemma}

\begin{proof}
First we show that the derivative of \(P\) is orderly. For each term in \(P\),  
	\[ \partial_{u_j}\left(\sigma^{(m_k)}(v + u^\top x)u^{\alpha_k}\right)
	= x_j\sigma^{(m_k + 1)}(v + u^\top x)u^{\alpha_k} 
		+ \sigma^{(m_k)}(v + u^\top x)\partial_{u_j}u^{\alpha_k}, \]
	where the second term may be zero. Thus the terms in \(\partial_{u_j} P\) remain orderly since
	\[ m_k + 1 - |\alpha_k| = m_k + 1 - \ord(u^{\alpha_k})
	= m_k - \left(\ord(u^{\alpha_k}) -1 \right) = m_k - \ord(\partial_{u_j} u^{\alpha_k}). \]
	Moreover, we have
	\[ \partial_{u_j} P = \sum_{k=1}^{K} \left[ a_k x_j\sigma^{(m_k + 1)}(v + u^\top x)u^{\alpha_k} 
		+ a_k \sigma^{(m_k)}(v + u^\top x)\partial_{u_j}u^{\alpha_k} \right]. \]
    Next we show that it is non-degenerate after combining like terms. This can be shown 
	by tracking the term with the maximal property. Specifically, let \(\alpha_1\) satisfy
	\[ |\alpha_1| = \max\left\{ |\alpha_1|,\cdots, |\alpha_K| \right\}, \]
	and we claim that the term \(\sigma^{(m_1 + 1)}(v + u^\top x)u^{\alpha_1} \) has no 
	like term in \(\partial_{u_j} P\). Under the claim, its coefficient after combining liker terms is \(a_1x_j\), which is nonzero, and thus 
    \(\partial_{u_j} P\) is non-degenerate. To prove the claim, note that
	\[ u^{\alpha_1} \neq \partial_{u_j}u^{\alpha_k} \quad \forall  k\in [K] \]
	by the maximality of \(|\alpha_1|\). Also we  have \(u^{\alpha_1} \neq u^{\alpha_k}\)
	for all \(k\neq 1\) by the definition of standard \(\sigma\)-polynomials. Thus the claim is 
	proved. Moreover, we can also let \(\alpha_1\) satisfy
	\( \ord_{u_j}(u^{\alpha_1}) = \max\left\{ \ord_{u_j}(u^{\alpha_1}),\cdots, 
		\ord_{u_j}(u^{\alpha_K}) \right\}. \)
\end{proof}

Now, we prove the theorem in the cases \(N=1,2\) and the general case separately. In the case \(N=1\), let \(P_1\) be of the form 
	\[ P := P_1(v + u^\top x_1) = \sum_{k=1}^{K} a_k \sigma^{(m_k)}(v + u^\top x_1)u^{\alpha_k}. \]
	We first reduce \(P\) by dividing some monomial. Since \(P\) is in the standard form, we can write it as 
	\[ P = u^\alpha \wt{P} \]
	for some multi-index \(\alpha\) and reduced standard \(\sigma(v + u^\top x_1)\)-polynomial 
	\(\wt{P}\). By assumption, we have \(P\equiv 0\). Then, by the continuity of \(\sigma\) and its derivatives, we have \(\wt{P}\equiv 0\). Thus we get a reduced standard \(\sigma\)-polynomial \(\wt{P}\) which is identically zero, still 
denoted by \(P\). Let \(u_1 = 0\) in \(P(v + u^\top x_1)\). Then we arrive at
\[ \wt{P} := \sum_{k,~ \alpha_{k1}=0} a_k \sigma^{(m_k)}(v + \wt{u}^\top \wt{x}_1)u^{\wt{\alpha}_k} \equiv 0, \quad \forall \wt{u} = (u_2,\ldots,u_d)\in \R^{d-1}, v\in\R, \]
with \(\wt{\alpha}_k = (\alpha_{k2},\ldots,\alpha_{kd}) \) and \(\wt{x}_1 = (x_{12},\ldots,x_{1d})\). The remaining exponents \(\wt{\alpha}_k\) arise from the indices \(\alpha_k\) with \(\alpha_{k1}=0\)
and there must exist an index satisfying this condition since \(P\) is reduced. Note that only the index \(\alpha\) with a zero first component remains. Thus \(\wt{P}\) is an orderly and hence a standard \(\sigma(v + \wt{u}^\top \wt{x}_1)\)-polynomial in \(\wt{u}\). 
	By continuing the preceding process for \(\wt{P}\) (i.e., first reducing it by dividing suitable monomial 
    and then setting some \(u_j\) to zero), we finally obtain 
	\[ a_k \sigma^{(m_k)}(v + \overline{u}^\top \overline{x}_1)u^{\overline{\alpha}_k} \equiv 0 \]
	for all \(v\in \R\) and all \(\overline{u}\) in the Euclidean space of some dimension (which may be zero),
	from which we deduce 
	\(\sigma^{(m_k)}(v + \overline{u}^\top \overline{x}_1) \equiv 0. \)
	This implies \(\sigma^{(m_k)} \equiv 0 \), since  \(v\) is arbitrary. 

In the case \(N=2\), by assumption, we have 
	\[ w_1P_1 + w_2P_2 = w_1\sum_{k=1}^{K_1} a_{1k} \sigma^{(m_{1k})}(v + u^\top x_1)u^{\alpha_{1k}}
	 + w_2\sum_{k=1}^{K_2} a_{2k} \sigma^{(m_{2k})}(v + u^\top x_2)u^{\alpha_{2k}} \equiv 0. \]
	We employ the change of variables
	\[ \begin{bmatrix} \wt{v} \\ \wt{u} \end{bmatrix} = A\begin{bmatrix} v \\ u \end{bmatrix},
		\quad \mbox{with } A=\begin{bmatrix}
		1 & x_{21} & \cdots & x_{2d}\\
		0 & 1 & \cdots & 0 \\
		\vdots & \vdots & \ddots & 0 \\
		0 & 0 & \cdots & 1
	\end{bmatrix} = \begin{bmatrix}
		1 & x_2^\top \\ 0 & I
	\end{bmatrix} \in  \R^{(d+1)\times (d+1)}, \]
	where the matrix \(A\) is invertible. Let \(r:=[v,u^\top]^\top \in \R^{d+1}\) 
	and \(y:=[1,x^\top]^\top \in \R^{d+1}\) be the extended coordinate and point. 
	Then the monomials \(u^\alpha\mapsto \wt{u}^\alpha\) and the term 
	\(\sigma^{(m)}(v + u^\top x)\) is transformed into
	\[ \sigma^{(m)}(r^\top y) = \sigma^{(m)}\left((A^{-1}\wt{r})^\top y\right) 
	= \sigma^{(m)}\left(\wt{r}^\top (A^{-\top}y)\right) =: \sigma^{(m)}(\wt{r}^\top \wt{y})
	=: \sigma^{(m)}(\wt{v} + \wt{u}^\top\wt{x}). \]
	Therefore, a standard \(\sigma(v + u^\top x)\)-polynomial in \(u\) becomes 
	a standard \(\sigma(\wt{v} + \wt{u}^\top\wt{x})\)-polynomial in \(\wt{u}\) with the coefficients and 
	monomials unchanged. Then we obtain
	\[ w_1\sum_{k=1}^{K_1} a_{1k} \sigma^{(m_{1k})}(\wt{v} + \wt{u}^\top \wt{x}_1)\wt{u}^{\alpha_{1k}}
	   + w_2\sum_{k=1}^{K_2} a_{2k} \sigma^{(m_{2k})}(\wt{v})\wt{u}^{\alpha_{2k}} \equiv 0, \quad
     \forall \wt{v}\in \R,\wt{u}\in \R^d. \]
    There must exist some \(j\in [d]\) such that
	\(\wt{x}_{1j}\neq 0\), otherwise \(\wt{y}_1\) is parallel to \(\wt{y}_2\).
	However, this is impossible because \(\wt{y}_1=A^{-\top}y_1, \wt{y}_2 = A^{-\top}y_2\)
	and the extended vectors \(y_1\) and \(y_2\) are not parallel by the distinction of \(x_1\) and \(x_2\).
    By differentiating sufficiently many times, say \(M\) times, 
	with respect to \(\wt{u}_j\) on both sides, we arrive at
	\[ w_1 \partial_{\wt{u}_j}^{~M} \wt{P}_1 = w_1 \partial_{\wt{u}_j}^{~M} 
	\sum_{k=1}^{K_1} a_{1k} \sigma^{(m_{1k})}(\wt{v} + \wt{u}^\top \wt{x}_1)\wt{u}^{\alpha_{1k}} \equiv 0.  \] 
	Since \(\partial_{\wt{u}_j}^{~M} \wt{P}_1\) is again a standard 
	\(\sigma(\wt{v} + \wt{u}^\top x_1)\)-polynomial in 
	\(\wt{u}\) by Lemma \ref{L:StandardForm}, the problem reduces to the case \(N=1\)
	and the desired conclusion follows. 
    
In the general \(N\) case, we have the equation
	\begin{equation}\label{E:Lin_Rel-sig_P} 
	    w_1\sum_{k=1}^{K_1} a_{1,k} \sigma^{(m_{1,k})}(v + u^\top x_1)u^{\alpha_{1,k}} + \cdots 
	+  w_N\sum_{k=1}^{K_N} a_{N,k} \sigma^{(m_{N,k})}(v + u^\top x_N)u^{\alpha_{N,k}} \equiv 0.
	\end{equation}
	We utilize the point \(x_N\) and perform the change of variables
	\[ \begin{bmatrix} \wt{v} \\ \wt{u} \end{bmatrix} = A\begin{bmatrix} v \\ u \end{bmatrix},
		\quad \mbox{with } A=\begin{bmatrix}
		1 & x_{N,1} & \cdots & x_{N,d}\\
		0 & 1 & \cdots & 0 \\
		\vdots & \vdots & \ddots & 0 \\
		0 & 0 & \cdots & 1
	\end{bmatrix} = \begin{bmatrix}
		1 & x_N^\top \\ 0 & I
	\end{bmatrix} \in  \R^{(d+1)\times (d+1)}, \]
	where the matrix \(A\) is invertible. Then we have
\begin{equation}\label{E:Post-Trans}	    w_1\sum_{k=1}^{K_1} a_{1,k} \sigma^{(m_{1,k})}(\wt{v} + \wt{u}^{\top} \wt{x}_1)\wt{u}^{\alpha_{1,k}} + \cdots 
	   + w_N\sum_{k=1}^{K_N} a_{N,k} \sigma^{(m_{N,k})}(\wt{v} + \wt{u}^\top \wt{x}_N)\wt{u}^{\alpha_{N,k}} \equiv 0.
	\end{equation}
	There must be some \(j\in \{1,\ldots,d\}\) such that
	\(\wt{x}_{1j}\neq 0\). Then by differentiating sufficiently many times, say \(M\) times, 
	with respect to \(\wt{u}_j\) on both sides of \eqref{E:Post-Trans}, we obtain
	\[ w_1 \partial_{\wt{u}_j}^{~M} \wt{P}_1 + \cdots + w_{N-1} \partial_{\wt{u}_j}^{~M} \wt{P}_{N-1}
	 =  w_1 \partial_{\wt{u}_j}^{~M} \sum_{k=1}^{K_1} a_{1,k}
	\sigma^{(m_{1,k})}(\wt{v} + \wt{u}^{\top} \wt{x}_1)\wt{u}^{\alpha_{1,k}} + \cdots \equiv 0.  \]
	This equation is nontrivial since the term \(\partial_{\wt{u}_j}^{~M} \wt{P}_1\) is 
	non-degenerate by the choice of \(j\). It reduces to the case smaller than \(N\) after discarding the zero terms. This completes the induction step.

\begin{remark}
	In the base case \(N=1\), the final remaining term determines an 
	upper bound of the order of \(\sigma\) as a polynomial. This term has many 
	choices according to the specific operation and can be found by the forked tree. 
    \begin{figure}[h]
	\centering
	\begin{tikzpicture}
	  [red, level distance=20mm,
	   every node/.style={inner sep=2pt},
	   level 1/.style={sibling distance=40mm},
	   level 2/.style={sibling distance=15mm},
	   level 3/.style={sibling distance=20mm}]
	  \node[red] (root) {P} [->]
	    child[black] {node {P(1)} 
	      child[black] {node {P(1,2)}}
	      child[black] {node {P(1,4)}}
	    }
	    child {node[red] {P(2)}
	      child[black] {node {P(2,1)}}
	      child {node {P(2,3)}
	        child[black] {node {P(2,3,1)}}
	        child {node {P(2,3,4)}}
	      }
	      child[black] {node {P(2,4)}}
	    }
	    child[black] {node {P(3)}}
	    child[black] {node {P(4)}
	      child {node {P(4,1)}}
	      child {node {P(4,2)}}
	      child {node {P(4,3)}}
	    };
	    \draw (root) -> node[above left] {\(u_2\)}  (root-2);
	    \draw (root-2) -> node[left] {\(u_3\)} (root-2-2);
	    \draw (root-2-2) -> node[above right] {\(u_4\)} (root-2-2-2);
	\end{tikzpicture}
	\begin{align*}
		P & = \left\{ u_1^2 u_3 u_4, u_1 u_2 u_4^2, u_1 u_3^2, u_2 u_3, u_3 u_4^2 \right\}
		\xrightarrow{\min \ord_{u_2}} P(2) = \left\{u_1^2 u_3 u_4, u_1 u_3^2, u_3 u_4^2\right\}\\
		& \hspace{3cm} \xrightarrow{\min \ord_{u_3}} P(2,3) = \left\{ u_1^2 u_3 u_4, u_3 u_4^2 \right\}
		\xrightarrow{\min \ord_{u_4}} P(2,3,4) = \left\{ u_1^2 u_3 u_4 \right\}
	\end{align*}
	\caption{The case $N=1$}
	\label{F:basecase}
	\end{figure}
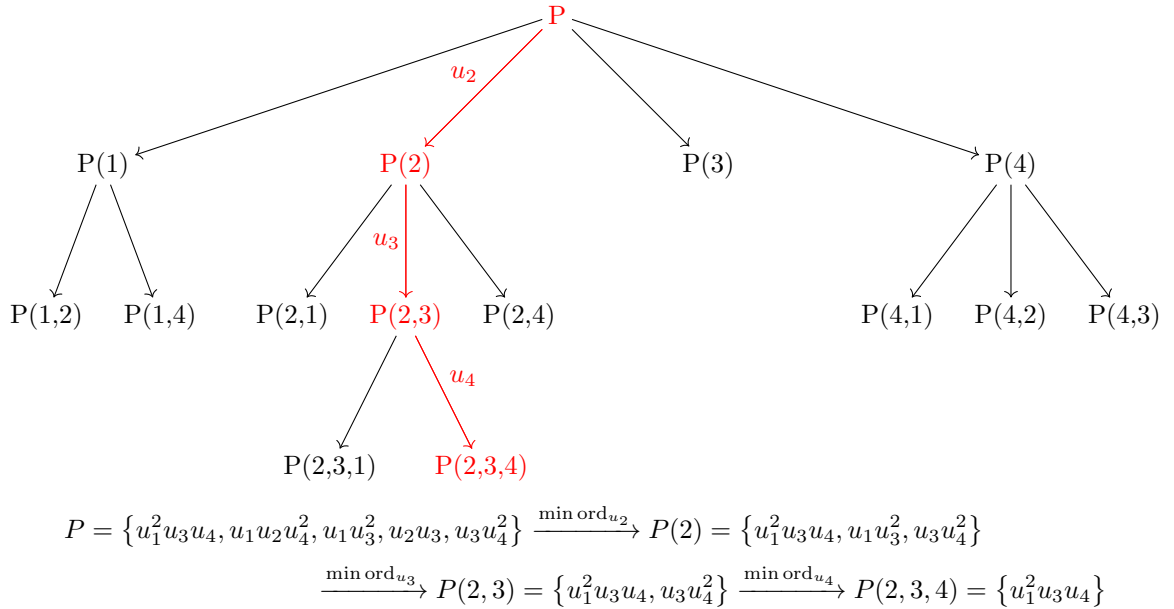
    The root of the forked tree is the set that consists of all the (pairwise different) 
	monomials \(\{ u^{\alpha_1},\ldots,u^{\alpha_K} \}\) in \(P\). It has (at most) \(d\) child 
	nodes, where the \(i\)-th node consists of the set of monomials which have the minimal 
	order in \(u_i\) provided that the variable \(u_i\) still exists after reducing these monomials,
	otherwise the \(i\)-th node is an empty set. 
	This step corresponds to taking some \(u_j\) to be zero in the proof.
    By continuing the bifurcation, we eventually reach an end such that 
    the node consists of exactly one monomial, since all the monomials in the expression of \(P\) 
    are pairwise different by the standard form of \(P\).

	For the case 
	\( P = \left\{ u_1^2 u_3 u_4, u_1 u_2 u_4^2, u_1 u_3^2, u_2 u_3, u_3 u_4^2 \right\} \),
	we illustrate the process by Fig. \ref{F:basecase}. 
	The inclusion relation \(P(i,j,k)\subset P(i,j) \subset P(i) \subset P\) denotes that 
    each set is obtained from its neighbouring superior set by extracting the terms with the minimal order in \(u_k, ~u_j\) and \(u_i\), respectively. When we go down along the red arrows, we successively consider the order of the variables \(u_2, ~u_3\) and \(u_4\), and finally arrive at one monomial.
    
    The trick of taking coordinate transformation and then differentiation is essentially equivalent to 
    taking a directional derivative. Since the points \(x_i\) are pairwise distinct, which implies the 
	extended points \((x_i,1)\) are pairwise non-parallel, we can always find a suitable direction
	and take directional derivative to annihilate at least one term in the equation.
\end{remark}

\subsection{Proof of Theorem \ref{T:LI-smo_u}}

	The base case \(N=1\) is very different from Theorem \ref{T:LI-smooth_uv}, since we have no extra variable \(v\). Suppose 
	\[  P := P_1(u^\top x_1) = \sum_{k=1}^{K} a_k \sigma^{(m_k)}( u^\top x_1)u^{\alpha_k} \equiv 0.  \]
	Let \(u = yz\) for some constant \(y\in\R^d\) and one-dimensional variable \(z\in\R\). Then we have 
	\begin{equation} \sum_{k=1}^{K} a_k y^{\alpha_k} \sigma^{(m_k)}( y^\top x_1\cdot z )z^{|\alpha_k|} = 0,
		\quad \forall  z\in \R. \label{eqn:eq-y}
\end{equation}
	We first choose \(y\) to ensure that \eqref{eqn:eq-y} is non-degenerate, which requires 
	\(y^\top x_1 \neq 0\) and not all the coefficients vanish after combining like terms.
	Since all \(a_k\) are nonzero, such a \(y\) exists.
	Since \(P\) is orderly, we can multiply or divide some power of \(z\) so that \(m_k=|\alpha_k|\)
	for all \(1\leq k \leq K\). Then after rescaling the variable \(z\), we obtain a homogeneous 
	Cauchy-Euler equation for \(\sigma\). 

In the general case, the proof is nearly identical with that of Theorem \ref{T:LI-smooth_uv}. That is, we first change the variables and then differentiate the resulting identity to reduce the degree until the case \(N=1\). The difference is that the transform matrix there is upper triangular with diagonal elements one, and thus the  \(\sigma\)-polynomials remain standard under transformation. In the absence of  \(v\), we can use the following linear transformation instead
	\[ \wt{u} = Au, \quad A=[e_1, \cdots, e_{j-1}, a, e_{j+1},\cdots, e_d]^\top \in \R^{d\times d}, \]
	where \(a\in \R^d\), located at \(j\)-th position with \(a_{j}\neq 0\), represents some point 
	of \(x_1,\ldots,x_N\). Then, we have \( u = A^{-1}\wt{u}=:T\wt{u} \),
	i.e., the substitution
	\[ u_j \mapsto t_1 \wt{u}_1 + \cdots + t_n \wt{u}_d \;(t_j\neq 0), 
		\quad u_i\mapsto \wt{u}_i, \; i\neq j. \]
By substituting	the change of variables into the standard \(\sigma(u^\top x)\)-polynomial in \(u\), we have 
	\[ P(u^\top x)  = \sum_{k=1}^{K} a_k \sigma^{(m_k)}(u^\top x)u^{\alpha_k} \mapsto 
	   P(\wt{u}^\top \wt{x}) = \sum_{k=1}^{K} a_k \sigma^{(m_k)}(\wt{u}^\top \wt{x})(T\wt{u})^{\alpha_k}, \]
     with \(\wt{x} = A^{-\top}x\).
	The \(\sigma\left(\wt{u}^\top \wt{x}\right)\)-polynomial \(P(\wt{u}^\top \wt{x})\) in \(\wt{u}\) 
	is still orderly after expansion. For the non-degeneracy, we track the term
	with the maximal order in \({u_j}\) as in Lemma \ref{L:StandardForm}. Specifically, let 
    \[ T=[e_1, \cdots, e_{j-1}, t, e_{j+1},\cdots, e_d]^\top \in \R^{d\times d}, \quad 
     \text{ with } ~ t_j = a_j^{-1}\neq 0, \]
    Then \((T\wt{u})_i = \wt{u}_i\) for \(i\in [d]\setminus j\) and \((T\wt{u})_j = t^\top \wt{u}\). Thus
    \[ (T\wt{u})^{\alpha_k} = \wt{u}_1^{\alpha_{k,1}}\cdots\wt{u}_{j-1}^{\alpha_{k,j-1}}
    (t^\top \wt{u})^{\alpha_{k,j}}\wt{u}_{j+1}^{\alpha_{k,j+1}}\cdots \wt{u}_d^{\alpha_{k,d}}.  \]
    Let \(\alpha_1\) satisfy
	\[ \ord_{u_j}(u^{\alpha_1}) = \max\left\{ \ord_{u_j}(u^{\alpha_1}),\cdots, 
		\ord_{u_j}(u^{\alpha_K}) \right\} \quad \Longleftrightarrow \quad 
        \alpha_{1,j} = \max\left\{ \alpha_{1,j},\cdots, 
		\alpha_{K,j} \right\}. \]
	Then the term \(\sigma^{(m_1)}(\wt{u}^\top \wt{x})\wt{u}^{\alpha_1} \) has no like term in 
    \(P(\wt{u}^\top \wt{x})\). Therefore, its coefficient after combining liker terms is given by 
    \(a_1 t_j^{\alpha_{1,j}}\neq 0\) and thus \(P(\wt{u}^\top \wt{x})\) is non-degenerate.

\subsection{Proof of Theorem \ref{T:LI-general_uv}}

The proof of the theorem requires several preliminary results. First, we prove that the
finite difference operation is commutative, similar to taking derivatives for smooth functions. 
\begin{lemma}\label{L:Comm-FD}
	Given two increments $p=(p_1,\ldots,p_d)$, $q=(q_1,\ldots,q_d)\in\R^d$, the operators
	\(\Delta_{p}\) and \(\Delta_{q}\) are commutative under composition:
	\[ \Delta_{q}\circ \Delta_{p} = \Delta_{p} \circ \Delta_{q}: \FF_d  \longrightarrow \FF_d. \]
\end{lemma}
\begin{proof}
    For any \(p,q\in\R^d\), we have 
	\begin{align*}
		\Delta_{q}\circ \Delta_{p} f & = \Delta_{q}\left(f(x+p)-f(x)\right) \\
		& = \left(f(x+p+q) - f(x+q)\right) - \left(f(x+p) - f(x)\right) \\
		& = \left(f(x+q+p) - f(x+p)\right) - \left(f(x+q) - f(x)\right) \\
		& = \Delta_{p}\left(f(x+q)-f(x)\right) = \Delta_{p} \circ \Delta_{q}f.
	\end{align*} 
\end{proof}

Given \(K\) increments \(p_1,\ldots, p_K\in \R^d\), let \(\bm{p} = (p_1,\ldots,p_K) \). 
We denote the \(K\)-times difference $\Delta_{\bm{p}}^K$ by
\[ \Delta_{\bm{p}}^K = \prod_{k=1}^{K}\Delta_{p_k} := \Delta_{p_1}\circ \cdots \circ \Delta_{p_K}. \]
If \(p_k = p\) for all \(k=1,\ldots,K\). We denote by \(\Delta_{p}^K\) the \(K\)-times difference \(\prod_{k=1}^{K}\Delta_{p} \). Next, we derive the transformation of the finite difference under the change of variables.
\begin{lemma}\label{L:FD-trans} 
Let \(A\in \mathbb{R}^{d\times d}\) be  invertible, \(\bm{p} = (p_1,\ldots,p_K) \) be \(K\) increments in \(\R^d\) and \( A\bm{p} := (Ap_1,\ldots,Ap_K) \). 
	Suppose that under the change of variables \(y=Ax\) \( (\text{or } x=A^{-1}y)\), the function 
	\(f(x)\in\FF_d\) becomes \(g(y) = f(A^{-1}y)\). Then
	the function  \( \Delta_{\bm{p}}^K f(x) \) becomes \(\Delta_{A\bm{p}}^Kg(y)\):
	\[ f(x)\xrightarrow{x=A^{-1}y}g(y) \quad \Longrightarrow \quad
		\Delta_{\bm{p}}^K f(x)\xrightarrow{x=A^{-1}y}\Delta_{A\bm{p}}^K g(y). \]
\end{lemma}
\begin{proof}
	When \(K=1\), we have for any increment \(p\in\R^d\), 
	\begin{align*}
		\Delta_p f(x) & = f(x+p) - f(x)\xrightarrow{x=A^{-1}y} f(A^{-1}y+p) - f(A^{-1}y) \\
		& = f\left( A^{-1}(x+Ap) \right) - f(A^{-1}y) 
		= g(y+Ap) - g(y) = \Delta_{Ap} g(y).
	\end{align*}
	In general, for \(K>1\), let \(\wt{\bm{p}} = (p_2,\ldots,p_K) \). Then by induction
$\Delta_{\wt{\bm{p}}}^{K-1} f(x)\xrightarrow{x=A^{-1}y}\Delta_{A\wt{\bm{p}}}^{K-1} g(y).$ Then, the desired assertion follows as:
	\begin{align*}
		\Delta_{\bm{p}}^K f(x)  = \Delta_{p_1}\left(\Delta_{\wt{\bm{p}}}^{K-1} f(x)\right)
		\xrightarrow{x=A^{-1}y} \Delta_{Ap_1}\left(\Delta_{A\wt{\bm{p}}}^{K-1} g(y)\right)
		= \Delta_{A\bm{p}}^K g(y).
	\end{align*}
\end{proof}

\begin{proof}[Proof of Theorem \ref{T:LI-general_uv}]

The base case \(N=1\) is identical with that in Theorem \ref{T:LI-smooth_uv}. In the general case, we have	
\[ w_1 P_1(v + u^\top x_1) + \cdots + w_N P_N(v + u^\top x_N) = 0. \]
	We choose some point not equal to \(x_1\), say \(x_N\), for the change of variables 
	\[ \begin{bmatrix} v^{(1)} \\ u^{(1)} \end{bmatrix} = A_{1}\begin{bmatrix} v \\ u \end{bmatrix},
		\quad A_{1} = \begin{bmatrix}
		1 & x_N \\ 0 & I
	\end{bmatrix} \in  \R^{(d+1)\times (d+1)}. \]
	Then we have
	\[ w_1P_1\left(v^{(1)} + u^{(1)}\cdot x^{(1)}_1\right) + \cdots 
	 + w_{N-1}P_{N-1}\left(v^{(1)} + u^{(1)}\cdot x^{(1)}_{N-1}\right)
	 + w_NP_N\left(v^{(1)}\right) \equiv 0. \]
	There must exist some \(j_1\in \{1,\ldots,d\}\) such that \(x^{(1)}_{1,j_1}\neq 0\).
Let \(\{ e_1,\ldots,e_d \}\subset \R^{d+1}\) be the standard Cartesian basis for the \(u\)-variable.  
Let \(p_1 = h_1e_{j_1}\). Then we employ the finite difference \(\Delta_{p_1}\) for reduction. Specifically, for some integer \(n_1\), there holds
\begin{equation}\label{E:FD-one_step} 
		w_1 \Delta_{p_1}^{n_1} P_1\left(v^{(1)} + u^{(1)}\cdot x^{(1)}_1\right) + \cdots 
	 	+ w_{N-1} \Delta_{p_1}^{n_1} P_{N-1}\left(v^{(1)} + u^{(1)}\cdot x^{(1)}_{N-1}\right)
	 	\equiv 0.
	\end{equation}
By the commutativity of finite difference in Lemma \ref{L:Comm-FD} and Lemma \ref{L:FD-trans}, after changing variables using \(x^{(1)}_{N-1}\) and taking difference, the identity \eqref{E:FD-one_step} is transformed into
	\[ w_1 \Delta_{p_2}^{n_2}\Delta_{A_{2}p_1}^{n_1} P_1\left(v^{(2)} + u^{(2)}\cdot x^{(2)}_1\right) + \cdots 
	+ w_{N-2} \Delta_{p_2}^{n_2} \Delta_{A_{2}p_1}^{n_1} P_{N-2}\left(v^{(2)} + u^{(2)}\cdot x^{(2)}_{N-2}\right)
	\equiv 0. \]
By repeating the argument for \(M\leq N-1\) steps, we  get
	\begin{equation}\label{E:FD-final_step} 
		\Delta_{p_M}^{n_M}\Delta_{A_{M}p_{M-1}}^{n_{M-1}}\cdots \Delta_{A_{M}\cdots A_{2}p_1}^{n_1} 
		P_1\left(v^{(M)} + u^{(M)}\cdot x^{(M)}_1\right) \equiv 0,
	\end{equation} 
	where for \(k=1,\ldots,M\), the increment \(p_k = h_ke_{j_k}\) for the one-dimensional increment \(h_k\in \R\) and some fixed 
	\(e_{j_k}\in \{ e_1,\ldots,e_d \}\), and each \(A_{k}\) has the form
	\[ A_{k} = \begin{bmatrix}
		1 & *_k \\ 0 & I
	\end{bmatrix} \in  \R^{(d+1)\times (d+1)}. \]
	Moreover, for \(k\in [M]\), we have
	\[ y_k := \begin{bmatrix} 1 \\ x^{(k)}_1 \end{bmatrix} = A_{k}^{-\top}
		\begin{bmatrix} 1 \\ x^{(k-1)}_1 \end{bmatrix} = A_{k}^{-\top} y_{k-1}, \]
	where \(x^{(0)}_1:= x_1\), and \(x^{(k)}_{1,j_k}\neq 0\), i.e., 
	\(y_k^\top e_{j_k}\neq 0\), by the choice of \(e_{j_k}\).
	Therefore, for any \(k\in [M-1]\), we have
$$ y_M^\top (A_{M}\cdots A_{k+1}p_k) = y_k^\top p_k = h_k y_k^\top e_{j_k} \neq 0, $$
and thus equation \eqref{E:FD-final_step} is non-degenerate. In the case \(P_1=\sigma\), we have 
$$\Delta_{p}\sigma\left(a_1z_1+\ldots+a_d z_d\right) 
    	= \left(\Delta_{a^\top p}\sigma\right)\left(a_1z_1+\ldots+a_d z_d\right),
$$
    for any increment \(p\in\R^d\) and point \(a\in\R^d\). Then, equation \eqref{E:FD-final_step} reduces to a one-dimensional 
    non-degenerate difference equation:
	\[ \Delta_{y_M^\top p_M}^{n_M}\Delta_{y_M^\top A_{M}p_{M-1}}^{n_{M-1}}\cdots 
		\Delta_{y_M^\top A_{M}\cdots A_{2}p_1}^{n_1} \sigma = 
		\Delta_{h_M x^{(M)}_{1,j_M}}^{n_M}\Delta_{h_{M-1} x^{({M-1})}_{1,j_{M-1}}}^{n_{M-1}}\cdots 
		\Delta_{h_1 x^{(1)}_{1,j_1}}^{n_1} \sigma \equiv 0. \]
	By the arbitrariness of \(h_1,\ldots h_M\) and Lemma \ref{L:PolyCha} below, \(\sigma\) is a polynomial. 
\end{proof}

\begin{remark}
In the proof, we get equation \eqref{E:FD-one_step} after reduction once. If the operator therein is differential instead of difference, we can expand the expression and get another standard \(\sigma\)-polynomial by Lemma \ref{L:StandardForm} as in the proof of Theorem \ref{T:LI-smooth_uv}. However, we obtain a high-dimensional difference-differential equation when carrying out the difference equation \eqref{E:FD-one_step} as the \(\sigma\)-polynomials \(P_1,\ldots,P_N\) involve differentials of \(\sigma\), which is generally intractable \citep{Elaydi:2005DE}. Therefore, we continue the reduction and finally get \eqref{E:FD-final_step}. That is, we transform equation \eqref{E:LinRelation} and thus the positivity of the DNTK into a difference-differential equation \eqref{E:FD-final_step}.
\end{remark}

\begin{lemma}[Theorem 4, \cite{Carvalho:2025Positivity}] \label{L:PolyCha} 
	Let \( f : \mathbb{R} \to \mathbb{R} \) be a continuous function. If for some positive 
	\( K \in \mathbb{Z} \) such that \( \Delta_{h}^{K}f(x) = 0 \) for all \(h\) and \(x\) in \(\R\). 
	Then \( f \) is a polynomial of order at most \( K-1 \).
\end{lemma}

\subsection{Proof of Theorem \ref{T:LI-RePU_u}}

	For each \(i\in[N]\), the \(\sigma\)-polynomial \(P_i\) takes the form 
	\[ P_i := P_i(u^\top x_i) = \sum_{k=1}^{K_i} a_{i,k}~(u^\top x_i)^{q_{i,k}}
	\mathbb{I}(u^\top x_i) u^{\alpha_{i,k}},  \]
	where all the coefficients \(a_{i,k}\) are nonzero, \(q_{i,k}\geq 0\) are integers, \(\mathbb{I}(\cdot)\)
	is shorthand of the indicator function \(\mathbb{I}_{>0}(\cdot)\), \(u^{\alpha_{i,k}}\) are pairwise different monomials for \(k\in[K_i]\)
    and the orderly property of \(P_i\) reads
    \[ q_{i,1} + |\alpha_{i,1}| = \cdots = q_{i,K_i} + |\alpha_{i,K_i}| .\] 
    Let \(\mathcal{L}^{d}\) denote the Lebesgue measure in \(\R^d\)
	and \(S_i \subset \R^d\) be the hyperplane
	$S_i := \{ u\in \R^d : u^\top x_i = 0 \}$.
	Then for \(j\neq i\), we have
	\[ \mathcal{L}^{d-1}(S_i\cap S_j) = 0  \quad \Longrightarrow \quad 
	 \mathcal{L}^{d-1}\Big(S_i\cap \bigcup_{j\neq i} S_j\Big) = 0.  \]
	Fix any \(i\in [N]\). Without loss of generality, suppose that there is some \(1\leq K\leq K_i\) such that
	\[ q_{i,1} =  \cdots =  q_{i,K} = q := \min\{q_{i,1}, \ldots , q_{i,K_i}\} \quad \text{ and } 
	\quad  q_{i,k} > q \text{ for } k > K. \]
	For any \( v\in S_i\setminus \cup_{j\neq i} S_j \) and \(r>0\), let \(B(v,r)\) be 
	the open ball centered at \(v\) with radius \(r\) and let 
	\begin{gather*}
		B^+_r:= \left\{ u\in B(v,r): u^\top x_i > 0 \right\} \quad \text{ and } \quad
		B^-_r:= \left\{ u\in B(v,r): u^\top x_i < 0 \right\}.
	\end{gather*}
	For \(j\neq i\), since \(v\notin S_j\) and \(S_j\) is closed, we have \( S_j\cap B(v,r) = \varnothing \) 
	for sufficiently small \(r\). Thus \(P_j(u^\top x_j)\) is smooth within every small ball. Then by Lebesgue differentiation theorem, we have
	\[ \lim_{r\to 0} \mathrel{\int\!\!\!\!\!\!-}_{B^+_r} \nabla_u^{q} P_j(u^\top x_j) \,\d u 
		= \lim_{r\to 0} \mathrel{\int\!\!\!\!\!\!-}_{B^-_r} \nabla_u^{q} P_j(u^\top x_j) \,\d u
		= \nabla_u^{q} P_j(v^\top x_j) \in \R^{d^q}, \]
	where the average integral means
	\( \mathrel{\int\!\!\!\!\!\!-}_{B^+_r} := \frac{1}{\mathcal{L}^{d}(B^+_r)} \int_{B^+_r}. \)
	For the case \(i\), we have \(P_i(u^\top x_i) =0 \) for all \(u\in {B^-_r},r>0\). Consequently, 
	\[ \mathrel{\int\!\!\!\!\!\!-}_{B^-_r} \nabla_u^{q} P_i(u^\top x_i) \,\d u = 0 \in \R^{d^q}. \]
	For the upper hemisphere \(B^+_r\), since \(v^\top x_i =0\), if \(1\leq k\leq K\), then
    \(q_{i,k} = q\) and we have  
	\[ \lim_{r\to 0} \mathrel{\int\!\!\!\!\!\!-}_{B^+_r} \nabla_u^{q} \left((u^\top x_i)^{q_{i,k}}
		\mathbb{I}(u^\top x_i) u^{\alpha_{i,k}}\right)  \,\d u 
        = \lim_{r\to 0} \mathrel{\int\!\!\!\!\!\!-}_{B^+_r} \nabla_u^{q} \left((u^\top x_i)^{q}_+
		 u^{\alpha_{i,k}}\right)  \,\d u 
		= q! x_i^{\otimes q} v^{\alpha_{i,k}}.  \]
	While for \(K < k \leq K_i\) we have \(q_{i,k}>q\) and the function 
    \((u^\top x_i)^{q_{i,k}}\mathbb{I}(u^\top x_i)\) is \(q\)-times continuously differentiable. Then 
    by Lebesgue differentiation theorem again, we have
	\[ \lim_{r\to 0} \mathrel{\int\!\!\!\!\!\!-}_{B^+_r} \nabla_u^{q} \left((u^\top x_i)^{q_{i,k}}
		\mathbb{I}(u^\top x_i) u^{\alpha_{i,k}}\right)  \,\d u 
        = \left. \nabla_u^{q} \left((u^\top x_i)^{q_{i,k}}
		\mathbb{I}(u^\top x_i) u^{\alpha_{i,k}}\right) \right|_{u=v} = 0.  \]
	Therefore, 
	\begin{align*}
		0 & =  \lim_{r\to 0} \mathrel{\int\!\!\!\!\!\!-}_{B^+_r}  \nabla_u^{q}\left(w_1P_1 + \cdots + w_NP_N\right) {\rm d}u
				- \mathrel{\int\!\!\!\!\!\!-}_{B^-_r} \nabla_u^{q}\left(w_1P_1 + \cdots + w_NP_N\right) {\rm d}u\\
		  & = w_i \lim_{r\to 0} \mathrel{\int\!\!\!\!\!\!-}_{B^+_r} \nabla_u^{q} P_i(u^\top x_i) \,\d u
		    = w_i \sum_{k=1}^{K} a_{i,k}q! v^{\alpha_{i,k}} x_i^{\otimes q}.
	\end{align*}
	We can choose \(v\in S_i\setminus \cup_{j\neq i} S_j\) properly so that 
\(P(v):=\sum_{k=1}^{K} a_{i,k}q! v^{\alpha_{i,k}}\) is nonzero. Then we get 
\(w_i = 0\), since \(x_i\) is nonzero. If no such \(v\) exists, then the (homogeneous) polynomial \(P(v)\) vanish on the whole hyperplane \(S_i\) by its continuity (note that 
the \(\mathcal{L}^{d-1}\) measure of the intersection \(S_i\cap\cup_{j\neq i} S_j\) is zero in 
the hyperplane \(S_i\) and the closure of the complement of a zero-measure set is the entire space), 
which implies that \(P(v)\) is identically zero in the whole space \(\R^d\)  by its analyticity. This is a contradiction since all its coefficients are nonzero.

\section{Proof in Section \ref{S:DNTK-deep}}\label{S-A:L-layer NN}

\subsection{Proof of Proposition \ref{P:deep-ifyNN}}
The proof is based on mathematical induction on the depth of the DNN. For the case \(L=1\), i.e., no hidden layer but only input and output layers and thus no limit, by definition we have
	\[ f^{(1)}(x;\theta) =  W^{(1)} x + \gamma b^{(1)}. \]
	We use \(a_i(x)\) to denote the \(i\)-th component of \(D\AA\), i.e., the 
	coefficient of \(\AA\) corresponding to the index \(\partial_{i}\) for \(i\in[d]\), while \(a(x)\)
	the zeroth order coefficient.  
	For any \(\mu \in [d_1]\) and \(x^\AA\in\PLDR_d\), we have 
    \begin{equation}\label{E:first-layer}
	\begin{aligned}
		f_{\mu}^{(1)}\left(x^\AA;\theta\right) 
		& = \AA\left( W^{(1)}_{\mu} x + \gamma b^{(1)}_{\mu} \right) = \sum_{|\alpha|\leq 1} 
			a_\alpha(x)\partial^\alpha\left( W^{(1)}_{\mu} x + \gamma b^{(1)}_{\mu} \right)\\
		& = \sum_{i=1}^{n} a_i(x) W^{(1)}_{\mu,i} + a(x)\left( W^{(1)}_{\mu} x + \gamma b^{(1)}_{\mu} \right)\\
		& =  W^{(1)}_{\mu} \Big(a(x)x+D\AA(x)\Big) + \gamma a(x) b^{(1)}_{\mu},
	\end{aligned}        
    \end{equation}
	which is a Gaussian variable as all the components of \(W^{(1)}_{\mu}\) and \(b^{(1)}_{\mu}\) are i.i.d. Gaussians. The covariance between the variables \(f_{\mu}^{(1)}\left(x^\AA;\theta\right) \)
	and \(f_{\mu}^{(1)}\left(y^\BB;\theta\right) \) is given by 
	\[ \Cov\left(x^\AA, y^\BB \right) = 
	\left(a(x)x+D\AA(x)\right)^\top \left(b(y)y+D\BB(y)\right) + \gamma^2 a(x)b(y).  \]
	Given any set of points \(X=\{x_i^{\AA_i}: i\in [N]\}\), by the expression of \(f_{\mu}^{(1)}(x^\AA)\) in \eqref{E:first-layer}, every linear combination of \(f^{(1)}_\mu(x_i^{\AA_i})\) is also Gaussian. Thus \(f^{(1)}_\mu(X)\) is a Gaussian vector and \(f^{(1)}_\mu\) is a Gaussian process on \(\PLDR_d\).

	Next, we analyze \(\ell+1\) layers with \(\ell\geq 1\). For \(\mu \in [d_{\ell+1}]\) and \(x^\AA\in\PLDR_d\), we have 
	\begin{align*}
		f_{\mu}^{(\ell+1)}\left(x^\AA;\theta\right) & =  \mathcal{A}\left( \frac{1}{\sqrt{d_\ell}} 
		W^{(\ell+1)}_{\mu} \sigma \left( f^{(\ell)}(x;\theta_{\leq \ell}) \right) 
		+ \gamma b^{(\ell+1)}_{\mu} \right) \\
		& = \frac{1}{\sqrt{d_\ell}} \sum_{i=1}^{d_\ell} W^{(\ell+1)}_{\mu,i} \AA \sigma 
		\left( f^{(\ell)}_i(x;\theta_{\leq \ell}) \right) + \gamma a(x) b^{(\ell+1)}_{\mu}.
	\end{align*}
	Note that by Fa\`a di Bruno's formula \eqref{E:Faa-alpha},
	\begin{equation}\label{E:l-layer AA_sig}
		\AA \sigma\left( f^{(\ell)}(x) \right) 
		  = \sum_{\alpha \in  I_\AA } a_\alpha (x) \sum_{\pi \in \Pi_\alpha} 
			\sigma^{(|\pi|)}\left( f^{(\ell)} \left( x \right) \right) 
			\prod_{A \in \pi} \partial^A f^{(\ell)}\left(x\right). 
	\end{equation}
	When \(d_1,\ldots d_{\ell-1}\to \infty\) sequentially, by induction, we have 
	\begin{equation}\label{E:AA-sigma-f_ify}
		\lim_{d_1,\ldots,d_{\ell-1}\to\infty} \AA \sigma \left( f^{(\ell)}_i(x;\theta_{\leq \ell}) \right)
		= \sum_{\alpha \in  I_\AA } a_\alpha (x) \sum_{\pi \in \Pi_\alpha} 
			\sigma^{(|\pi|)}\left( f^{(\ell)}_{i,\infty} \left( x \right) \right) 
			\prod_{A \in \pi}  f^{(\ell)}_{i,\infty}\left(x^{\partial^A}\right). 
	\end{equation}
    Therefore, in the case \(d_1,\ldots d_{\ell-1}\to \infty\) sequentially, we have
	\[ f_{\mu}^{(\ell+1)}\left(x^\AA;\theta\right) 
	= \frac{1}{\sqrt{d_\ell}} \sum_{i=1}^{d_\ell} W^{(\ell+1)}_{\mu,i} \AA \sigma 
		\left( f^{(\ell)}_{i,\infty}(x;\theta_{\leq \ell}) \right) + \gamma a(x) b^{(\ell+1)}_{\mu}. \]
The first term is a normalized sum of i.i.d. random variables. By the identity \eqref{E:AA-sigma-f_ify}, the Cauchy-Schwarz inequality, the finite variance
	of \(f^{(\ell)}_{i,\infty}\) and the polynomial growth of \(\sigma\) and its derivatives, we have the mean square integrability
	\[ \E\left[\left( W^{(\ell+1)}_{\mu,i} \AA \sigma 
		\left( f^{(\ell)}_{i,\infty}(x;\theta_{\leq \ell}) \right) \right)^2 \right]
		= \E\left[\left( \AA \sigma 
		\left( f^{(\ell)}_{i,\infty}(x;\theta_{\leq \ell}) \right) \right)^2 \right] < \infty. \]
	Given a finite set of points \(X = \{ x_1^{\AA_1},\ldots,x_N^{\AA_N} \}\subset \PLDR_d\),
	by the central limit theorem in Lemma \ref{lem:MCLT} when \(d_{\ell}\) tends to infinity, we have 
	\begin{equation*}
		\lim_{d_\ell\to\infty} \frac{1}{\sqrt{d_\ell}} \sum_{i=1}^{d_\ell} \left[ W^{(\ell+1)}_{\mu,i}  
		\mathcal{A}_j \sigma \left( f^{(\ell)}_{i,\infty}(x_j;\theta_{\leq \ell}) \right) \right]_{1\leq j\leq N}
		\overset{d}{\longrightarrow} \mathcal{N}_N(0,\Sigma),
	\end{equation*}
	where the covariance at the pair \((x^\AA,y^\BB)\) is given by
	\begin{align*}
		\Sigma\left(x^\AA,y^\BB\right)  = \E_{f^{(\ell)}_{\infty}}
		\left[ \AA \sigma \left( f^{(\ell)}_{\infty}(x) \right) 
		\BB \sigma \left( f^{(\ell)}_{\infty}(y) \right)\right].
	\end{align*}
	Let \(a_j(x)\) be the zeroth order term of \(\AA_j\) for \(1\leq j\leq N\). 
	Since the sum of independent Gaussian variables is still a Gaussian, combining the bias term gives
	\begin{align*}
		\lim_{d_\ell\to\infty} f_{\mu}^{(\ell+1)}(X) \overset{d}{\longrightarrow} 
		\mathcal{N}_N(0,\Sigma) + \gamma b^{(\ell+1)}_{\mu} \left[a_j(x_j)\right]_{1\leq j\leq N}
		= \mathcal{N}_N\left(0,\Sigma^{(\ell+1)}_{\infty}(X)\right).
	\end{align*} 
	This shows that \(f_{\mu}^{(\ell+1)}\) is a Gaussian process on \(\PLDR_d\) with the covariance function given by \(\Sigma^{(\ell+1)}_{\infty}\) and thus finishes the induction.
    The independence of different components of \(f^{(L)}_\infty\) follows identically as Proposition \ref{P:2-ifyNN} and thus omitted.

\subsection{Proof of Theorem \ref{T:deep-ifyDNTK}}
The proof is also based on mathematical induction. 
	For the case \(L=1\), there is no limit to take. Given 
	\(\mu,\nu\in[d_1]\) and \(x^\AA, y^\BB\in\PLDR_d\), we have 
	\begin{align*}
		\Theta^{(1)}_{\mu\nu}\left( x^\AA,y^\BB \right) 
		& = \llan \nabla_\theta \AA \left(W^{(1)}_\mu x + \gamma b^{(1)}_\mu\right), 
			\nabla_\theta \BB \left(W^{(1)}_\nu y + \gamma b^{(1)}_\nu\right) \rran \\
		& = \nabla_\theta \left( W^{(1)}_{\mu} \left(a(x)x+D\AA(x)\right) + \gamma a(x) b^{(1)}_{\mu} \right)\cdot \\
        & \hspace{4cm}
			\nabla_\theta \left( W^{(1)}_{\nu} \left(b(y)y+D\BB(y)\right) + \gamma b(y) b^{(1)}_{\nu} \right) \\
		& = \delta_{\mu\nu} \left(a(x)x+D\AA(x)\right)^\top \left(b(y)y+D\BB(y)\right) 
			+ \delta_{\mu\nu} \gamma^2 a(x)b(y) \\
		& = \delta_{\mu\nu} \Sigma^{(1)}_{\infty}\left( x^\AA,y^\BB \right).
	\end{align*}

Like in the proof of Proposition \ref{P:deep-ifyNN}, for the \(\ell+1\) layers with \(\ell\geq 1\), by 
letting \(d_1,\ldots, d_{\ell-1}\to \infty\) sequentially, we get a Gaussian process \(f^{(\ell)}_\infty\). 
When \(d_\ell\to\infty\), we arrive at a deterministic kernel \(\Theta^{(\ell+1)}_\infty\) by the law of large numbers. Specifically, given \(\mu \in [d_{\ell+1}]\) and \(x^\AA,y^\BB\in\PLDR_d\), 
we split the gradient \(\nabla_{\theta}\) into two parts \(\nabla_{\theta_{\ell+1}}\) and \(\nabla_{\theta_{\leq\ell}}\) as in Theorem \ref{T:2-ifyDNTK}:
\begin{align*}	\Theta^{\ell+1}_{\mu\nu}\left( x^\AA,y^\BB \right) 
= \llan \nabla_\theta \mathcal{A}f^{\ell+1}_\mu(x;\theta), \nabla_\theta\mathcal{B}f^{\ell+1}_\nu(y;\theta)  \rran
=: \textrm{I} + \textrm{II},
\end{align*}
with the terms ${\rm I}$ and ${\rm II}$ given respectively by
\begin{align*}
\textrm{I} & = \llan \nabla_{\theta_{\ell+1}} \mathcal{A}f^{(\ell+1)}_\mu(x;\theta),\nabla_{\theta_{\ell+1}}\mathcal{B}f^{(\ell+1)}_\nu(y;\theta)  \rran,\\ 
{\rm II} &=  \llan \nabla_{\theta_{\leq \ell}} \mathcal{A}f^{(\ell+1)}_\mu(x;\theta), \nabla_{\theta_{\leq\ell}}\mathcal{B}f^{(\ell+1)}_\nu(y;\theta)  \rran.
\end{align*}
For the first term \textrm{I}, by the law of large numbers, 
	\begin{align*}
		\textrm{I} & = \llan \nabla_{\theta_{\ell+1}} \mathcal{A}f^{(\ell+1)}_\mu(x;\theta), 
				\nabla_{\theta_{\ell+1}}\mathcal{B}f^{(\ell+1)}_\nu(y;\theta)  \rran \\
		& =  \llan \nabla_{\theta_{\ell+1}} \frac{1}{\sqrt{d_\ell}} W^{(\ell+1)}_{\mu} 
				\AA \sigma \left( f^{(\ell)}(x) \right) + \AA \gamma b^{(\ell+1)}_\mu , 
			\nabla_{\theta_{\ell+1}} \frac{1}{\sqrt{d_\ell}} W^{(\ell+1)}_{\nu} 
				\BB \sigma \left( f^{(\ell)}(y) \right) + \BB \gamma b^{(\ell+1)}_\nu \rran  \\
		& = \frac{1}{d_\ell} \sum_{k=1}^{d_{\ell+1}} \sum_{i=1}^{d_\ell}
		 	\delta_{k \mu} \AA \sigma \left( f^{(\ell)}_i (x) \right)
			\delta_{k \nu} \BB \sigma \left( f^{(\ell)}_i (y) \right)
			+ \delta_{\mu \nu} \gamma^2 a(x)b(y) \\
		& = \delta_{\mu \nu} \frac{1}{d_\ell} \sum_{i=1}^{d_\ell}
		 	\AA \sigma \left( f^{(\ell)}_i (x) \right)
			\BB \sigma \left( f^{(\ell)}_i (y) \right)
			+ \delta_{\mu \nu} \gamma^2 a(x)b(y) \\
		& = \delta_{\mu \nu} \frac{1}{d_\ell} \sum_{i=1}^{d_\ell}
		 	\AA \sigma \left( f^{(\ell)}_{i,\infty} (x) \right)
			\BB \sigma \left( f^{(\ell)}_{i,\infty} (y) \right)
			+ \delta_{\mu \nu} \gamma^2 a(x)b(y) 
			\qquad \text{as } d_1,\ldots,d_{\ell-1} \to \infty \\
		& \overset{p}{\longrightarrow} \delta_{\mu \nu} \Sigma^{(\ell+1)}_{\infty}\left( x^\AA,y^\BB \right)
			\qquad \text{as } d_{\ell}\to \infty.
	\end{align*}
For the second term \textrm{II}, given any element \(\rho \in \theta_{\leq\ell}\),
	by the continuous differentiability of the activation \(\sigma\), we have
	\begin{align*}
		\partial_\rho\left( \AA f^{(\ell+1)}_{\mu}(x;\theta)\right)
		& = \partial_\rho\left( \frac{1}{\sqrt{d_\ell}} W^{(\ell+1)}_{\mu} 
			\AA \sigma \left( f^{(\ell)}(x;\theta_{\leq\ell}) \right) + \AA b^{(\ell+1)}_\mu \right) \\
		& = \frac{1}{\sqrt{d_\ell}} W^{(\ell+1)}_{\mu} \AA \partial_\rho 
			\sigma \left( f^{(\ell)}(x;\theta_{\leq\ell}) \right)  \\
		& = \frac{1}{\sqrt{d_\ell}} \sum_{i=1}^{d_\ell} W^{(\ell+1)}_{\mu,i}
			\AA \left(\sigma' \left( f^{(\ell)}_i(x;\theta_{\leq\ell})  \right) 
			\partial_\rho\left( f^{(\ell)}_i(x;\theta_{\leq\ell}) \right)\right) \\
		& = \frac{1}{\sqrt{d_\ell}} \sum_{i=1}^{d_\ell} W^{(\ell+1)}_{\mu,i}
			\sum_{\alpha\in I_\AA } a_\alpha(x) \sum_{\xi \leq \alpha}\binom{\alpha}{\xi}
			\partial^{\xi} \sigma' \left( f^{(\ell)}_i(x) \right) 
			\partial_\rho \partial^{\alpha-\xi} f^{(\ell)}_i(x;\theta_{\leq\ell}).
	\end{align*}
	Similarly, we have
	\begin{align*}
		\partial_\rho\left( \BB f^{(\ell+1)}_{\nu}(y;\theta)\right)
		& = \frac{1}{\sqrt{d_\ell}} \sum_{j=1}^{d_\ell} W^{(\ell+1)}_{\nu,j}
			\BB \left(\sigma' \left( f^{(\ell)}_j(y;\theta_{\leq\ell})  \right) 
			\partial_\rho\left( f^{(\ell)}_j(y;\theta_{\leq\ell}) \right)\right) \\
		& = \frac{1}{\sqrt{d_\ell}} \sum_{j=1}^{d_\ell} W^{(\ell+1)}_{\nu,j}
			\sum_{\beta\in I_\BB } b_\beta(y) \sum_{\zeta \leq \beta}\binom{\beta}{\zeta}
			\partial^{\zeta} \sigma' \left( f^{(\ell)}_j(y) \right) 
			\partial_\rho \partial^{\beta-\zeta} f^{(\ell)}_j(y;\theta_{\leq\ell}).
	\end{align*}
	Note that the components \(f^{(\ell)}_i(x;\theta_{\leq\ell})\) and \(f^{(\ell)}_j(y;\theta_{\leq\ell})\)
	have no common parameters if \(i\neq j\). Consequently,
	\begin{align*}
		\textrm{II} & = \sum_{\rho\in \theta_{\leq\ell}} \partial_\rho\left( \AA f^{(\ell+1)}_{\mu}(x;\theta)\right)
		\partial_\rho\left( \BB f^{(\ell+1)}_{\nu}(y;\theta)\right) \\
		& = \frac{1}{d_\ell} \sum_{i,j=1}^{d_\ell}  W^{(\ell+1)}_{\mu,i} W^{(\ell+1)}_{\nu,j} 
			\sum_{\rho\in \theta_{\leq\ell}}
			\AA \left(\sigma' \left( f^{(\ell)}_i(x)  \right) 
			\partial_\rho\left( f^{(\ell)}_i(x) \right)\right)
			\BB \left(\sigma' \left( f^{(\ell)}_i(y)  \right) 
			\partial_\rho\left( f^{(\ell)}_i(y) \right)\right) \\
		& = \frac{1}{d_\ell} \sum_{i,j=1}^{d_\ell}  W^{(\ell+1)}_{\mu,i} W^{(\ell+1)}_{\nu,j} 
			\sum_{\alpha\in I_\AA }\sum_{\beta\in I_\BB } a_\alpha(x)b_\beta(y)
			\sum_{\xi\leq\alpha}\sum_{\zeta\leq\beta}\binom{\alpha}{\xi}\binom{\beta}{\zeta} \cdot \\
		& \qquad\quad \partial^{\xi} \sigma' \left( f^{(\ell)}_i(x) \right) 
			\partial^{\zeta} \sigma' \left( f^{(\ell)}_j(y) \right) 
			\sum_{\rho\in \theta_{\leq\ell}}
			\partial_\rho \partial^{\alpha-\xi} f^{(\ell)}_i(x;\theta_{\leq\ell})
			\partial_\rho \partial^{\beta-\zeta} f^{(\ell)}_j(y;\theta_{\leq\ell}) \\
		& = \sum_{\alpha\in I_\AA }\sum_{\beta\in I_\BB } a_\alpha(x)b_\beta(y)
			\sum_{\xi\leq\alpha}\sum_{\zeta\leq\beta}\binom{\alpha}{\xi}\binom{\beta}{\zeta}
			\frac{1}{d_\ell} \sum_{i=1}^{d_\ell}  W^{(\ell+1)}_{\mu,i} W^{(\ell+1)}_{\nu,i} \cdot \\
		& 	\qquad\quad \partial^{\xi} \sigma' \left( f^{(\ell)}_i(x) \right) 
			\partial^{\zeta} \sigma' \left( f^{(\ell)}_i(y) \right)
			\sum_{\rho\in \theta_{\leq\ell}}
			\partial_\rho \partial^{\alpha-\xi} f^{(\ell)}_i(x;\theta_{\leq\ell})
			\partial_\rho \partial^{\beta-\zeta} f^{(\ell)}_i(y;\theta_{\leq\ell}).
	\end{align*}
The proof is completed using the law of large numbers and by induction and combining the expressions of \textrm{I} and \textrm{II}.

\subsection{Proof of Theorem \ref{T:Property (LI)}}

Suppose that the conclusion holds for \(\ell\geq 2\). In particular, it holds for \(\ell=2\). By equation \eqref{E:Property (LI)}, we know that \(\sigma\) cannot be a polynomial. In the case \(\ell+1\), for the vector 
\(w=[w_1,\ldots,w_N]^\top\in\R^{K_1+\cdots+K_N}\), we have 
\begin{align*}
	w^\top \Sigma^{(\ell+1)}_\infty(X) w = \E_{f^{(\ell)}_\infty} 
	\left[ \left\| D_1^{(\ell)}\left(X\right)w  \right\|^2 \right]
	+ \left\| \left[a_{11},\ldots,a_{1K_1},\cdots,a_{N1},\ldots,a_{NK_N}\right]w \right\|^2 \geq 0,
\end{align*}
where \(D_1^{(\ell)}\left(X\right)w\) is a function of the Gaussians given by \(f^{(\ell)}_\infty\)
such that 
\[ D_1^{(\ell)}\left(X\right)w = \sum_{i=1}^{N} \sum_{k=1}^{K_i}w_{ik} 
\sum_{\pi \in \Pi_{\alpha_{ik}}} \sigma^{(|\pi|)}\left( f^{(\ell)}_\infty  \left( x_i \right) \right) 
\prod_{A \in \pi} f^{(\ell)}_\infty \left(x_i^{\partial^A}\right). \]
First, we identify a basis of the Gaussians \(\{f^{(\ell)}_\infty  ( x_i ),
f^{(\ell)}_\infty (x_i^{\partial^A}): A\in\pi\in \Pi_{\alpha_{ik}}, 
1\leq k\leq K_i, i\in[N]\}\).
For each \(i\in[N]\), we select a set of different random variables \(R_i\) from 
\(\{f^{(\ell)}_\infty ( x_i ),
f^{(\ell)}_\infty (x_i^{\partial^A})\}\), which consists of \(f^{(\ell)}_\infty  \left( x_i \right)\) and all the distinct variables in the set \( \{f^{(\ell)}_\infty(x_i^{\partial^A}): A\in\pi\in \Pi_{\alpha_{ik}}, 1\leq k\leq K_i\}\). 
Then the disjoint union 
\(\sqcup_{i\in[N]} R_i\) forms a basis of these Gaussians, since they are independent by induction 
and any other Gaussian equals one of them by the choice of the set \(R_i\).

If \(w^\top \Sigma^{(\ell+1)}_\infty(X) w\) vanishes identically, then \(D_1^{(\ell)}\left(X\right)w\) is identically
zero on the Euclidean space of the dimension equal to the number of random variables in the set
\(\sqcup_{i\in[N]} R_i\). Let all the random variables \(f^{(\ell)}_\infty\left( x_i \right), f^{(\ell)}_\infty\left( x_i^{\partial^A} \right) \)
in \(R_i\) be zero for \(1\neq i\in [N]\). Then we have 
\begin{align*}
	\sum_{k=1}^{K_1}w_{1k} 
	\sum_{\pi \in \Pi_{\alpha_{1k}}} \sigma^{(|\pi|)}\left( f^{(\ell)}_\infty  \left( x_1 \right) \right) 
	\prod_{A \in \pi} f^{(\ell)}_\infty \left(x_1^{\partial^A}\right) \equiv 0
\end{align*}
as a function of variables in \(R_1\). If \(w_1\neq 0 ~ (\in\mathbb{R}^{K_1})\), then the left hand side is a standard \(\sigma\)-polynommial 
with \(v= f^{(L-1)}_\infty  \left( x_1 \right)\)
and the vector \(u\) being the other variables in \(R_1\).
However, by Theorem \ref{T:LI-smooth_uv}, this is impossible since \(\sigma\) is not a polynomial.
Thus \(w_1=0\) and similarly we can deduce \(w=0\), which finishes the proof of the theorem.

\subsection{Proof of Theorem \ref{T:cor_Pos_DNTK}}

    Now, we prove the positivity result of \(\Sigma^{(L)}_\infty(X)\) and \(\Theta^{(L)}_\infty(X)\). Let the zero of \(\AA\in\LD_d\), denoted by \(\mathcal{Z}(\AA)\), be the intersection of the zero set of all its coefficients:
    \[ \mathcal{Z}(\AA) := \bigcap_{\alpha\in I_\AA} \mathcal{Z}(a_\alpha(z)), \quad \text{ for }
    \AA = \sum_{\alpha\in I_\AA} a_\alpha(z)\partial^{\alpha}. \]
	We first consider the two-layer case \(L=2\).
	Since \(x_i\notin \mathcal{Z}(\AA_i)\) for all \(i\in[N]\), each term in \eqref{E:LinRelation}
	is a standard \(\sigma\)-polynomial in \(u\). By Theorems \ref{T:LI-smooth_uv} and 
	\ref{T:LI-RePU_u}, the linear dependence relation \eqref{E:LinRelation} cannot happen for 
	nonzero \(w\in\R^N\), hence \(\Sigma^{(2)}_{\infty}(X)\) and \(\Theta^{(2)}_{\infty}(X)\)
	are positive definite by the former discussion.

	For \(L\geq 3\), we first prove the positivity of \(\Sigma^{(L)}_\infty(X)\). Given any \(w\in\R^N\), we have
	\begin{align*}
		w^\top \Sigma^{(L)}_\infty(X) w = \E_{f^{(L-1)}_\infty} 
		\left[ \left\| D_1^{(L)}\left(X\right)w  \right\|^2 \right]
		+ \left\| \left[a_{1},\ldots,\ldots,a_{N}\right]w \right\|^2,
	\end{align*}
	where \(a_{i}\) for \(i\in[N]\) is the value of zeroth order term of 
	\(\AA_i\) at the point \(x_i\) and 
	\begin{align*}
		D_1^{(L)}\left(X\right)w = \sum_{i=1}^{N} w_i
		\sum_{\alpha \in  I_{\AA_i} } a_\alpha (x_i) \sum_{\pi \in \Pi_\alpha} 
		\sigma^{(|\pi|)}\left( f^{(L-1)}_\infty  \left( x_i \right) \right) 
		\prod_{A \in \pi} f^{(L-1)}_\infty \left(x_i^{\partial^A}\right). 
	\end{align*}
	The expectation \(\E_{f^{(L-1)}_\infty} \) is taken with respect to the Gaussians
	\(\{f^{(L-1)}_\infty ( x_i ), f^{(L-1)}_\infty (x_i^{\partial^A}): i\in[N], A\in\pi\in\Pi_\alpha, \alpha\in I_{\AA_i} \}\).
	We can choose one basis of this set as in Theorem \ref{T:Property (LI)}. For each \(i\in[N]\), choose a set of different Gaussians \(R_i\) from
	\(\{f^{(L-1)}_\infty ( x_i ), f^{(L-1)}_\infty (x_i^{\partial^A}): A\in\pi\in\Pi_\alpha, \alpha\in I_{\AA_i}\} \), which consists of \(f^{(\ell)}_\infty\left( x_i \right)\) and some \(f^{(\ell)}_\infty\left( x_i^{\partial^A} \right)\).
	Then, \(\sqcup_{i\in[N]} R_i\) forms a basis of these Gaussians. If \(\E_{f^{(L-1)}_\infty} [ \| D_1^{(L)}(X)w  \|^2]=0\),
	then \(D_1^{(L)}\left(X\right)w \equiv 0\) as a function of variables in \(\sqcup_{i\in[N]} R_i\). Again, let all the variables \(f^{(L-1)}_\infty \left( x_i \right), f^{(L-1)}_\infty \left(x_i^{\partial^A}\right)\) in \(\sqcup_{i\neq 1} R_i\) be zero. Then, we get
	\[ w_1\sum_{\alpha \in  I_{\AA_1} } a_\alpha (x_1) \sum_{\pi \in \Pi_\alpha} 
		\sigma^{(|\pi|)}\left( f^{(L-1)}_\infty  \left( x_1 \right) \right) 
		\prod_{A \in \pi} f^{(L-1)}_\infty \left(x_1^{\partial^A}\right) \equiv 0 \]
	as a function of variables in \(R_1\). 
	Since \(x_1\notin \mathcal{Z}(\AA_1)\), if \(w_1\neq 0\), the left side is a standard \(\sigma\)-polynommial with \(v= f^{(L-1)}_\infty  \left( x_1 \right)\)
	and the vector \(u\) being the other Gaussians in \(R_1\).
	By Theorem \ref{T:LI-smooth_uv}, \(\sigma\) is a polynomial, which contradicts the assumption. Hence we must have \(w_1=0\). Similarly, we can deduce \(w=0\), which implies the positivity of \(\Sigma^{(L)}_\infty(X)\).

For the positivity of \(\Theta^{(L)}_\infty(X)\), we rewrite the expression \eqref{E:deep-DNTK} as 
$$\Theta^{(L)}_\infty(X) = \Sigma^{(L)}_\infty(X) + K^{(L)}_\infty(X).$$ 
Then, it suffices to prove the positive semi-definiteness of the matrix \(K^{(L)}_\infty(X)\) since \(\Sigma^{(L)}_\infty(X)\) is positive definite. 
	We assert that for all \(L\geq 2\), there exists a finite nonempty index set \(I_L\) such that 
	the term \(K^{(L)}_\infty(x^\AA,y^\BB)\) can be wirtten as 
	\begin{equation}\label{E:exp-K^L}
		K^{(L)}_\infty(x^\AA,y^\BB) = \sum_{i\in I_L} \E\left[ \GG^{(L-1)}_i(x^\AA) \GG^{(L-1)}_i(y^\BB) \right],
	\end{equation}
	where \(\GG^{(L-1)}_i(x^\AA)\) 
	is a function of the Gaussian variables of the form \(f^{(\ell)}_{\infty}(x^{\partial^\alpha})\) 
	for \(1\leq\ell\leq L-1\), which is only related to \(x^\AA\) while independent of \(y^\BB\), and the expectation therein is taken with respect to all such normal variables. If so, the kernel matrix \(K^{(L)}_\infty(X)\) reads 
	\begin{align*}
		K^{(L)}_\infty(X) = \sum_{i\in I_L}\E\left[\GG^{(L-1)}_i(X)^\top \GG^{(L-1)}_i(X)\right]
        , \quad 
		\GG^{(L-1)}_i(X) = \left[\GG^{(L-1)}_i(x_1^{\AA_1}),\cdots, \GG^{(L-1)}_i(x_N^{\AA_N})\right],
	\end{align*}
	and thus for any \(w\in \R^N\), we get
	\[ w^\top K^{(L)}_\infty(X)w = \sum_{i\in I_L}\E\left[ \left\| \GG^{(L-1)}_i(X) w \right\|^2  \right]\geq 0. \] 
	Next, we prove the assertion for \(L\geq 2\) by mathematical induction. For \(\AA\in\LD_d\), we use \(a^\AA_0\) to denote its zeroth order term while \(a^\AA_k\) with \(k\in[d]\) 
	for the coefficients of the partial derivative \(\partial_{z_k}\). When \(L=2\), we have 
	\begin{align*}
		& K^{(2)}_\infty(x^\AA,y^\BB) 
         = \sum_{\alpha\in I_\AA }\sum_{\beta\in I_\BB } a_\alpha(x)b_\beta(y) 
		\sum_{\xi\leq\alpha}\sum_{\zeta\leq\beta}\binom{\alpha}{\xi}\binom{\beta}{\zeta}
		\dot{\Sigma}^{(2)}_{\infty}\left( x^{\partial^\xi},y^{\partial^\zeta} \right)
		\Theta^{(1)}_{\infty}\left( x^{\partial^{(\alpha-\xi)}},y^{\partial^{(\beta-\zeta)}} \right)\\
		& = \sum_{\alpha\in I_\AA }\sum_{\beta\in I_\BB } a_\alpha(x)b_\beta(y) 
		\sum_{\xi\leq\alpha}\sum_{\zeta\leq\beta}\binom{\alpha}{\xi}\binom{\beta}{\zeta}
		 \E\left[ \partial^\xi \sigma' \left( f^{(1)}_{\infty}(x) \right) 
		 \partial^\zeta \sigma' \left( f^{(1)}_{\infty}(y) \right)\right] \cdot \\
		& \qquad \left(\left(a_0^{\partial^{(\alpha-\xi)}}(x)x+D\partial^{(\alpha-\xi)}(x)\right)^\top 
		\left(b_0^{\partial^{(\beta-\zeta)}}(y)y+D\partial^{(\beta-\zeta)}(y)\right) 
			+ \gamma^2 a_0^{\partial^{(\alpha-\xi)}}(x)b_0^{\partial^{(\beta-\zeta)}}(y)\right)\\
		& = \sum_{k=1}^{d} \E\left[ \left( \sum_{\alpha\in I_\AA }\sum_{\xi\leq\alpha}a_\alpha(x)\binom{\alpha}{\xi}
			\partial^\xi \sigma' \left( f^{(1)}_{\infty}(x) \right) 
			\left(a_0^{\partial^{(\alpha-\xi)}}(x)x_k + a_k^{\partial^{(\alpha-\xi)}}(x)\right) \right)
			\right. \cdot \\
		& \qquad \qquad \qquad \left. \left( \sum_{\beta\in I_\BB }\sum_{\zeta\leq\beta}b_\beta(y)\binom{\beta}{\zeta}
			\partial^\zeta \sigma' \left( f^{(1)}_{\infty}(y) \right)
			\left(b_0^{\partial^{(\beta-\zeta)}}(y)y_k + a_k^{\partial^{(\beta-\zeta)}}(y)\right)  \right)
			 \right] \\
		& \qquad + \E\left[ \left(\gamma \sum_{\alpha\in I_\AA }\sum_{\xi\leq\alpha}a_\alpha(x)\binom{\alpha}{\xi}
			\partial^\xi \sigma' \left( f^{(1)}_{\infty}(x) \right) 
			a_0^{\partial^{(\alpha-\xi)}}(x) \right) \right. \\
        & \qquad \qquad \qquad \left.
			 \left(\gamma \sum_{\beta\in I_\BB }\sum_{\zeta\leq\beta}b_\beta(y)\binom{\beta}{\zeta}
			\partial^\zeta \sigma' \left( f^{(1)}_{\infty}(y) \right)
			b_0^{\partial^{(\beta-\zeta)}}(y) \right) \right].
	\end{align*}
	Suppose the assertion holds for \(\ell\geq 2\). Then we have 
	\begin{align*}
		 K^{(\ell+1)}_{\infty}\left( x^\AA,y^\BB \right) 
		& =  \sum_{\alpha\in I_\AA }\sum_{\beta\in I_\BB } a_\alpha(x)b_\beta(y)
			 \sum_{\xi\leq\alpha}\sum_{\zeta\leq\beta}\binom{\alpha}{\xi}\binom{\beta}{\zeta}
			\dot{\Sigma}^{(\ell+1)}_{\infty}\left( x^{\partial^\xi},y^{\partial^\zeta} \right)\cdot \\
		& \qquad \qquad \qquad 
            \left(\Sigma^{(\ell)}_{\infty}\left( x^{\partial^{(\alpha-\xi)}},y^{\partial^{(\beta-\zeta)}} \right)
		  + K^{(\ell)}_{\infty}\left( x^{\partial^{(\alpha-\xi)}},y^{\partial^{(\beta-\zeta)}}  \right) \right) \\
		& = \sum_{\alpha\in I_\AA }\sum_{\beta\in I_\BB } a_\alpha(x)b_\beta(y)
			\sum_{\xi\leq\alpha}\sum_{\zeta\leq\beta}\binom{\alpha}{\xi}\binom{\beta}{\zeta}
			\E\left[ \partial^\xi \sigma' \left( f^{(\ell)}_{\infty}(x) \right) 
			\partial^\zeta \sigma' \left( f^{(\ell)}_{\infty}(y) \right)\right] \cdot \\
		& \qquad \qquad \Biggl( \E\left[ \partial^{\alpha - \xi} \sigma \left( f^{(\ell-1)}_{\infty}(x) \right) 
			\partial^{\beta - \zeta} \sigma \left( f^{(\ell-1)}_{\infty}(y) \right)\right] 
			+ \gamma^2 a_0^{\partial^{(\alpha-\xi)}}(x)b_0^{\partial^{(\beta-\zeta)}}(y) \\
        & \qquad \qquad \qquad \qquad
            + \sum_{i\in I_\ell }\E\left[ \GG^{(\ell-1)}_i(x^{\partial^{(\alpha-\xi)}}) 
			\GG^{(\ell-1)}_i(y^{\partial^{(\beta-\zeta)}}) \right] \Biggr) \\
		& =  \E\left[ \left( \sum_{\alpha\in I_\AA }\sum_{\xi\leq\alpha}a_\alpha(x)\binom{\alpha}{\xi}
			\partial^\xi \sigma' \left( f^{(\ell)}_{\infty}(x) \right) 
			\partial^{\alpha - \xi} \sigma \left( f^{(\ell-1)}_{\infty}(x) \right) \right)
			\right. \cdot \\
		& \qquad \qquad \qquad \left. \left( \sum_{\beta\in I_\BB }\sum_{\zeta\leq\beta}b_\beta(y)\binom{\beta}{\zeta}
			\partial^\zeta \sigma' \left( f^{(\ell)}_{\infty}(y) \right)
			\partial^{\beta - \zeta} \sigma \left( f^{(\ell-1)}_{\infty}(y) \right)  \right)
			 \right] \\
		& \quad +  \E\left[ \left(\gamma \sum_{\alpha\in I_\AA }\sum_{\xi\leq\alpha}a_\alpha(x)\binom{\alpha}{\xi}
			\partial^\xi \sigma' \left( f^{(\ell)}_{\infty}(x) \right) 
			a_0^{\partial^{(\alpha-\xi)}}(x) \right) \right. \cdot \\
		& \qquad \qquad \qquad \left. \left(\gamma \sum_{\beta\in I_\BB }\sum_{\zeta\leq\beta}b_\beta(y)\binom{\beta}{\zeta}
			\partial^\zeta \sigma' \left( f^{(\ell)}_{\infty}(y) \right)
			b_0^{\partial^{(\beta-\zeta)}}(y) \right) \right] \\
		& \quad + \sum_{i\in I_\ell } \E\left[ \left( \sum_{\alpha\in I_\AA }\sum_{\xi\leq\alpha}a_\alpha(x)\binom{\alpha}{\xi}
			\partial^\xi \sigma' \left( f^{(\ell)}_{\infty}(x) \right) 
			\GG^{(\ell-1)}_i(x^{\partial^{(\alpha-\xi)}})  \right) \right. \cdot \\
		& \qquad \qquad \qquad \left. \left( \sum_{\beta\in I_\BB }\sum_{\zeta\leq\beta}b_\beta(y)\binom{\beta}{\zeta}
			\partial^\zeta \sigma' \left( f^{(\ell)}_{\infty}(y) \right)
			\GG^{(\ell-1)}_i(y^{\partial^{(\beta-\zeta)}}) \right) \right].
	\end{align*}
	This proves the assertion and thus finishes the proof.

\begin{remark}	
In the NTK case, the second summand \(K^{(L)}_\infty(X)\) is a Hadamard product of two positive semi-definite kernels, which directly yields the desired conclusion via the Schur product theorem. 
Note that the semi-definiteness of \(K^{(L)}_\infty\) is an inherent property of DNTK which does not depend on the assumptions on the points and activation. 
Similar to the two-layer case, each summand in the identity \eqref{E:exp-K^L} of \(K^{(L)}_\infty(x^\AA,y^\BB)\) can contribute to the positivity of the kernel matrix \(\Theta^{(L)}_\infty(X)\). Since the number of index in \(I_L\) increases as \(L\) grows, deeper NNs potentially make the DNTK more likely to be positive definite.
\end{remark}

\vskip 0.2in
\bibliographystyle{abbrv}
\bibliography{reference}

\end{document}